\documentclass[10pt,journal]{IEEEtran}
\usepackage{amsmath,amssymb,amsfonts,amsthm}
\usepackage{array}
\usepackage[caption=false,font=footnotesize,labelfont=rm,textfont=rm]{subfig}
\usepackage{textcomp}
\usepackage{stfloats}
\usepackage{url}
\usepackage{cases}

\newtheorem{theorem}{Theorem}
\newtheorem{lemma}{Lemma}
\newtheorem{corol}{Corollary}
\newtheorem{defi}{Definition}
\usepackage{multirow}
\usepackage{xcolor}
\usepackage{verbatim}
\usepackage{graphicx}
\usepackage{setspace}
\usepackage{cite}
\usepackage{threeparttable}
\usepackage{booktabs}
\allowdisplaybreaks[2]
\usepackage[ruled]{algorithm2e} 
\usepackage{accents}
\newcommand{\ubar}[1]{\underaccent{\bar}{#1}}

\usepackage{changepage}

\def\BibTeX{{\rm B\kern-.05em{\sc i\kern-.025em b}\kern-.08em
    T\kern-.1667em\lower.7ex\hbox{E}\kern-.125emX}}
\usepackage{balance}

\begin{document}
\bstctlcite{setting}

\title{Provable Performance Bounds for Digital Twin-driven Deep Reinforcement Learning in Wireless Networks: A Novel Digital-Twin Bisimulation Metric}

\author{Zhenyu~Tao,
        Wei~Xu,~\IEEEmembership{Fellow,~IEEE},
        and Xiaohu~You,~\IEEEmembership{Fellow,~IEEE}

\thanks{Z. Tao, W. Xu, and X. You are with the National Mobile Communications Research Lab, Southeast University, Nanjing 210096, China, and also with the Pervasive Communication Research Center, Purple Mountain Laboratories, Nanjing 211111, China (email: \{zhenyu\_tao, wxu, xhyu\}@seu.edu.cn).} %
}


\maketitle


\begin{abstract}
Digital twin (DT)-driven deep reinforcement learning (DRL) has emerged as a promising paradigm for wireless network optimization, offering safe and efficient training environment for policy exploration. However, in theory existing methods cannot always guarantee real-world performance of DT-trained policies before actual deployment, due to the absence of a universal metric for assessing DT's ability to support reliable DRL training transferrable to physical networks. In this paper, we propose the DT bisimulation metric (DT-BSM), a novel metric based on the Wasserstein distance, to quantify the discrepancy between Markov decision processes (MDPs) in both the DT and the corresponding real-world wireless network environment. We prove that for any DT-trained policy, the sub-optimality of its performance (regret) in the real-world deployment is bounded by a weighted sum of the DT-BSM and its sub-optimality within the MDP in the DT. Then, a modified DT-BSM based on the total variation distance is also introduced to avoid the prohibitive calculation complexity of Wasserstein distance for large-scale wireless network scenarios. Further, to tackle the challenge of obtaining accurate transition probabilities of the MDP in real world for the DT-BSM calculation, we propose an empirical DT-BSM method based on statistical sampling. We prove that the empirical DT-BSM always converges to the desired theoretical one, and quantitatively establish the relationship between the required sample size and the target level of approximation accuracy. Numerical experiments validate this first theoretical finding on the provable and calculable performance bounds for DT-driven DRL.

\end{abstract}

\begin{IEEEkeywords}
Digital twin, Markov decision process (MDP), deep reinforcement learning (DRL), transfer learning, bisimulation metric.
\end{IEEEkeywords}

\section{Introduction}
\IEEEPARstart{T}{he} long-term evolution of cellular networks, marked by growing scale, density, and heterogeneity, substantially increases the difficulty of wireless network optimization \cite{10183795}. Deep reinforcement learning (DRL) emerges as a promising solution for tackling extensive state and action spaces and nonconvex optimization problems. It has been successfully applied to various network optimization tasks, such as admission control \cite{van2019optimal}, resource allocation \cite{10552627}, node selection \cite{9767557}, and task offloading \cite{10024766} in wireless networks.

Training DRL agents in real-world wireless networks faces notable obstacles, including prohibitive trial-and-error costs and poor network performance before convergence \cite{9372298}. To overcome these issues, the concept of digital twin (DT) has been introduced \cite{9854866}. By creating a virtual replica of real-world networks, DT provides safe and efficient training environment for exploring policies, i.e., state-to-action mappings, before real-world deployment \cite{tao2023wireless}. This approach enables effective DRL training while ensuring the performance of the physical wireless network \cite{10623528}. By now, DT-driven DRL paradigm has been widely employed in numerous network optimization tasks \cite{10438215,10415196,10363344}, offering significant advantages over traditional training methods, such as reduced performance fluctuations \cite{10078846}, accelerated convergence \cite{10486201}, and lower energy consumption \cite{10345669}.

The effectiveness of the DT-driven DRL paradigm relies not only on the performance of DRL models, as extensively studied in prior works \cite{9904958}, but more fundamentally on the quality of the underlying DT. This quality corresponds to the DT's capability to provide a reliable training environment and ensure deployment performance. However, there remains a lack of universal, systematic, and quantitative metrics to assess such quality in current studies. A critical question arises: How can we determine whether a DT is sufficiently effective for DRL training? One intuitive approach is to evaluate the similarity between DT and its real-world counterpart using task-specific metrics. For example, user trajectory and wireless channel similarities can be used for DT-enabled user association tasks \cite{tao2024parallel}, and service request arrival rate can be compared for DT-enabled admission control tasks \cite{tao2023DTAC}. However, these metrics are developed for specific tasks and do not serve as a universal measure of the DT's quality. They capture specific aspects of DT fidelity to the real-world scenario, which inadequately characterize the entire environment, and their relationship with the performance of transferred policies remains unclear.

An alternative approach is to evaluate the DT's quality by assessing the performance of a DT-trained policy in the real environment. For instance, in \cite{10234388}, the effectiveness of the DT-driven DRL method is demonstrated in terms of deployment performance, compared with a policy trained directly in the real environment and other baseline methods (random and greedy). However, challenges arise with the deployment-based assessment. First, the performance of the transferred policy is often unstable, influenced not only by DT's quality but also by factors such as distinct DRL algorithms, neural network (NN) architectures, iteration counts, and other hyperparameters \cite{9372298}. This fluctuation in deployment performance further complicates precise evaluation of DT's quality. Meanwhile, the generalization ability of NNs \cite{pmlr-v97-cobbe19a} may obscure certain flaws in DTs, potentially causing problems when applying a seemingly reliable DT for training other models, such as lightweight NNs \cite{chen2024review} and distributed DRL agents~\cite{9725256}. Most importantly, since
DT’s quality is evaluated through real-world deployment, the performance of each transferred policy cannot be guaranteed during the evaluation. This undermines the primary goal of DT to ensure reliable and predictable performance of real-world wireless networks.

In addition, the DT-driven DRL can be considered a specific case of transfer learning within DRL, allowing us to resort to machine learning theories \cite{BSM}. While transfer learning in supervised learning has been extensively studied \cite{5288526,10636241}, its integration into the Markov decision process (MDP) framework introduces additional complexities. As a result, theoretical work on transfer learning in DRL remains largely underexplored, both in the early stage research \cite{taylor2009transfer} and more recent studies \cite{10172347}. We will discuss this in the following subsection of related works.

Consequently, for effective applications of DT-driven DRL, there is an urgent need for a direct, policy-independent, and quantitative assessment of DT's quality, as well as deeper understanding of its relationship with the performance of transferred policies in its real-world counterpart. The main contributions of the paper are as follows:


 


\begin{itemize}
\item{We propose the DT bisimulation metric (DT-BSM), a novel metric to measure the discrepancy between the two MDPs that model DT and the real wireless network environment, respectively. Extended from the bisimulation metric (BSM) \cite{BSM}, the DT-BSM quantifies the discrepancy of the optimal value functions in the couple of two MDPs, thereby ensuring policy-independent comparison.
}
\item{Through DT-BSM, we establish provable performance bounds for the policy transferred from the DT to the corresponding real-world wireless network. To mitigate the ``curse of dimensionality" in DT-BSM calculation for large-scale wireless networks, we further introduce a modified DT-BSM based on the total variation distance. This adaptation significantly reduces the computational complexity while maintaining the bound on deployment performance.
} 
\item{In case the transition probabilities of the MDP are not directly obtainable in real-world scenarios for DT-BSM calculation, we propose the empirical DT-BSM through statistical sampling. We prove that
the empirical DT-BSM always converges to the theoretical one and quantitatively determines the required sample size for a target level of approximation accuracy. }

\item{Numerical experiments on a typical admission control task in wireless networks corroborate the theoretical findings on the performance bound. To the best of our knowledge, DT-BSM is the first metric to directly measure the DT's quality for DRL training and to provide provable and calculable performance bounds for DT-driven DRL.
}
\end{itemize}



\subsection{Related Works}
The bisimulation metric, proposed by Ferns \textit{et al.} in the early 2000s \cite{BSM}, quantifies state similarity within an MDP through a smoothly varying distance measure with respect to rewards and transition probabilities. BSM has since been widely adopted for state aggregation \cite{Li2006TowardsAU,LeLan2021} and value function approximation \cite{NEURIPS2021_256bf8e6} in MDP. However, as BSM was originally designed for a single MDP, its application to transfer learning between distinct MDPs remains limited. Prior works have explored the use of BSM for policy transfer analysis in reinforcement learning (RL) but face critical limitations. For instance, The study in \cite{castro2010using} employed a relaxed definition of BSM to analyze policy transfer, but its theoretical bound is limited to transferring an optimal policy within the source MDP, an assumption that can be overly idealized for practical applications. Moreover, this bound is formulated solely for the one-step action-value function, which does not fully reflect the long-term performance of the transferred policy (see Theorem 5 in \cite{castro2010using}). Another study in \cite{phillips2006knowledge} merged the state spaces of the source and target MDPs into a disjoint union state space, thereby using BSM for theoretical analysis in transfer learning. This merging process incurs prohibitive computational costs that are impractical for RL with large state-action spaces, as highlighted in \cite{taylor2009transfer}, letting alone more complicated DRL tasks.
\setlength{\parskip}{0\baselineskip}

With advancements in high-performance computing devices and NN architectures, DRL has largely replaced traditional RL, becoming the predominant method for learning-based policy optimization \cite{9904958}. Over the past decade, DRL has primarily been applied to games and simulations, with research efforts largely dedicated to enhancing performance within given environments \cite{silver2016mastering,berner2019dota}. While some studies have explored transfer learning in DRL, they focus on improving policy performance through transfer learning rather than theoretically analyzing the policy transfer between source and target environments \cite{10172347}. Provable and computable performance bounds for policy transfer in DRL remain largely underexplored.

In wireless network optimization, existing theoretical studies associated with DT-driven DRL are mainly focused on the analysis of DRL. For example, the study in \cite{10147312} established gradient convergence bounds for multi-agent DRL in communication networks, and study \cite{10623365} theoretically compared the convergence rate of DRL methods in dynamic spectrum access. Although \cite{10287987} proved DRL’s capability to attain optimal policies in edge computing, theoretical analysis for provable performance of DT-driven DRL remains scarce. To our knowledge, one exception in \cite{10522623} attempted such analysis but limitd its bound to one-step transition probabilities, which inadequately reflects long-term DRL performance in deployment (see Theorem 1 and Appendix A in \cite{10522623}). The absence of theoretical analysis for DTs leads most existing DT-driven DRL studies to assume perfect environment replication by DT and to focus exclusively on enhancing DRL methods, as seen in applications in cell-free networks \cite{10078846}, space-air-ground integrated networks \cite{10345669}, and mobile edge computing \cite{10235311}. This idealized assumption limits DT-driven DRL’s practical applicability in real-world wireless networks, where modeling inaccuracies are inevitable.

\subsection{Outline}
The paper is organized as follows. Section~\ref{Sec:2} presents a detailed problem definition. In Section~\ref{Sec:3}, we provide a sketch of our main results. Section~\ref{sec:defi} provides the formal definition of the proposed DT-BSM and its associated properties. Section~\ref{sec:bound} establishes the provable performance bound for policy transfer. The empirical DT-BSM is introduced in Section~\ref{sec:emp}. Then we validate our results through numerical experiments in Section~\ref{sec:exp}. Finally, a conclusion is drawn in Section~\ref{sec:con}.

\section{Problem Description} \label{Sec:2}
Notable trial-and-error costs and suboptimal performance of DRL agents before convergence have driven the development of the DT-driven DRL paradigm. As illustrated in Fig.~\ref{fig:scene}, this paradigm involves collecting real-world data to construct a virtual replica, i.e., DT, of the physical environment. Instead of direct interaction with the physical network, DRL agents rely on the DT to explore and optimize the policy in a safer, faster, and more cost-efficient manner. After sufficient training, the DRL agents are then deployed into the real-world environment, avoiding risks and costs in the physical wireless network during the training phase. However, challenges persist in the current DT-driven DRL paradigm. Existing evaluation metrics for DTs typically assess their fidelity to the real environment from specific aspects, failing to establish direct connections with the optimization target. Consequently, the performance of transferred policies cannot be reliably guaranteed before actual deployment in practice.

\begin{figure}[t]
    \centering
    \includegraphics[width=0.85\linewidth]{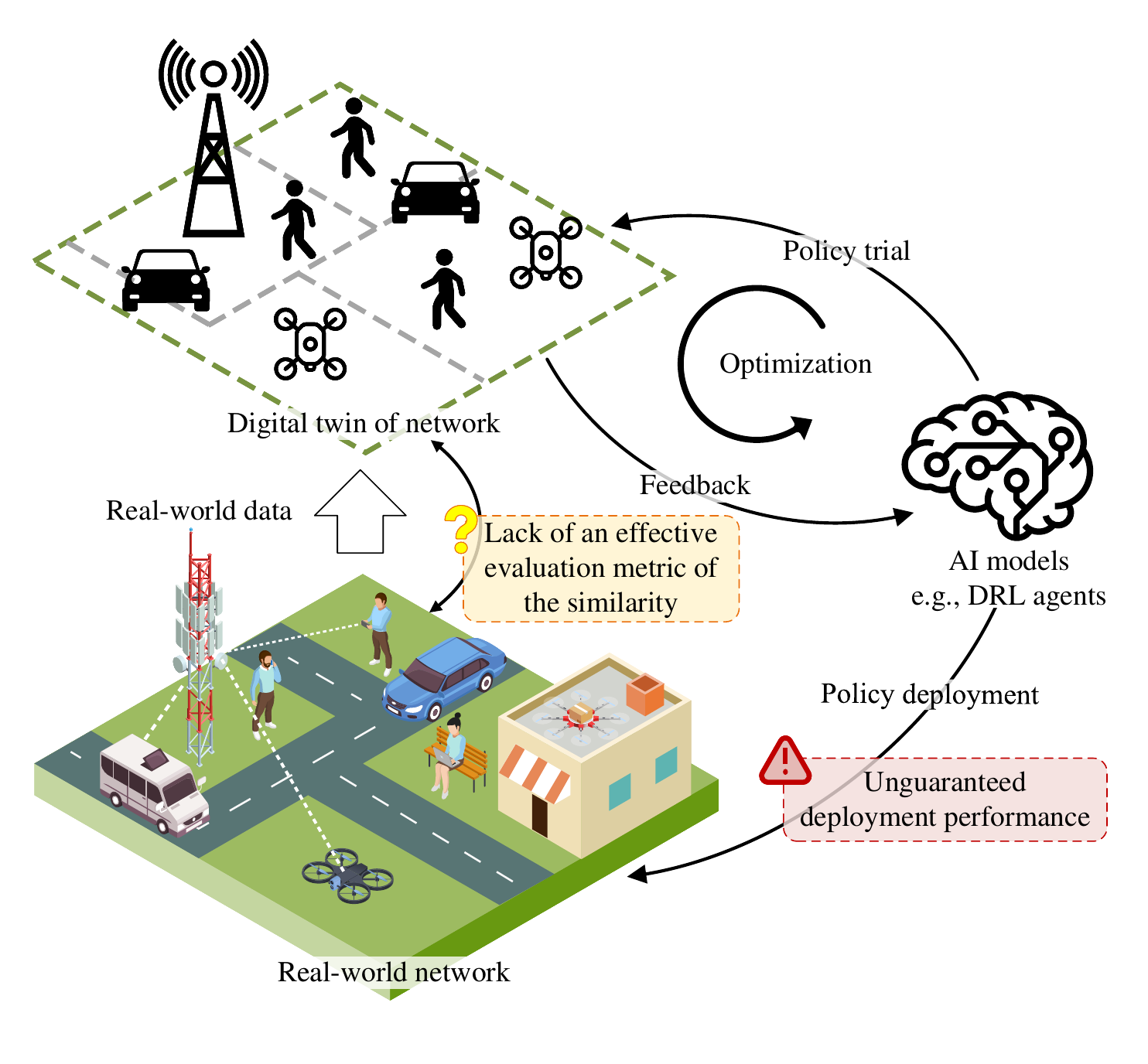}
    \caption{A schematic of DT-driven DRL}
    \label{fig:scene}
\end{figure}
We formulate this problem as follows. Consider a real-world network optimization task modeled by a 5-tuple MDP $ \langle \mathcal{S}, \mathcal{A}, \mathbb{P}, R, \gamma\rangle$, where $\mathcal{S}$ is the state space, $\mathcal{A}$ is the action space, $\mathbb{P}(\tilde{s}|s,a)$ is the transition probability ($a\in\mathcal{A}$, $\{\tilde{s},s\}\in\mathcal{S}$, and $\tilde{s}$ denotes the next state), $ R(s,a)$ is the reward function, and $\gamma \in \left(0,1\right)$ is the discount factor. For clarity, we refer to such an MDP as the real MDP in the following discussion. The goal of DRL is to devise a policy $\pi$ that maximizes the long-term average reward. In conventional DRL, the objective is achieved by optimizing the associated value function as follows
\begin{equation}
    \max_\pi\  V_\text{real}^\pi(s),\ \forall s \in \mathcal{S},
\end{equation}
where the value function is defined by the Bellman equation
\begin{equation}
    V_\text{real}^\pi(s)=R(s,\pi(s))+\gamma\sum_{\tilde{s}\in \mathcal{S}} \mathbb{P}(\tilde{s}|s,\pi(s)) V_\text{real}^\pi(\tilde{s}), \label{eq:v2}
\end{equation}
and $\pi(s)\in \mathcal{A}$ is the action chosen by policy $\pi$ at state $s$. To simplify the discussion, a deterministic policy $\pi$ is assumed. Nevertheless, our findings can be readily extended to stochastic policies by reformulating (\ref{eq:v2}) to incorporate a summation over action probabilities. When $\pi$ reaches the optimality, denoted by superscript $*$, the value function becomes the optimal value function, given by
\begin{equation}
    V_\text{real}^*(s)=\max_{a}\Big\{R(s,a)+\gamma\sum_{\tilde{s}\in \mathcal{S}}\mathbb{P}(\tilde{s}|s,a)V_\text{real}^*(\tilde{s})\Big\}. \label{eq:v1}
\end{equation}

The DT is modeled by a similar MDP $\langle \mathcal{S}, \mathcal{A}, \mathbb{P}', R', \gamma\rangle$, where $\mathcal{S}, \mathcal{A},$ and $\gamma$ can be readily aligned with the real MDP through simulation configuration. In contrast, $\mathbb{P}'$ and $R'$ inevitably differ from their real counterparts, with discrepancies ranging from minor to significant depending on the fidelity of the simulation. We denote this MDP as the DT MDP in the following discussion. The value function and optimal value function in DT, respectively denoted by $V_\text{DT}^\pi(s)$ and $V_\text{DT}^*(s)$, are defined by replacing $\mathbb{P}$ and $R$ in (\ref{eq:v2}) and (\ref{eq:v1}) with $\mathbb{P}'$ and $R'$, respectively. In DT-driven DRL, the policy $\pi$ is optimized by
\begin{equation}
    \max_\pi\  V_\text{DT}^\pi(s),\ \forall s \in \mathcal{S}
\end{equation}
before being deployed into the real MDP. The deployment performance of the transferred policy is evaluated through its sub-optimality (also known as the regret \cite{kaelbling1996reinforcement}), which is defined as the expected discounted reward loss when following the transferred policy $\pi$ instead of the optimal policy in the real MDP. Since the optimal expected discounted reward is given by $V_\text{real}^*(s)$, this loss is formulated as $V_\text{real}^*(s)-V_\text{real}^\pi(s)$.

Our goal is to devise a metric for directly quantifying the discrepancy between the DT MDP, $\langle \mathcal{S}, \mathcal{A}, \mathbb{P}', R', \gamma\rangle$, and the real MDP, $ \langle \mathcal{S}, \mathcal{A}, \mathbb{P}, R, \gamma\rangle$. This metric should, for any policy $\pi$ trained in DT, establish a provable upper bound on the sub-optimality of the transferred policy before actual deployment, thereby ensuring reliable and predictable performance of physical wireless networks.

\section{Main Results}\label{Sec:3}
In this section, we summarize our main results, the newly defined DT-BSM along with the upper bound on the sub-optimality of the transferred policy.

We draw on the BSM \cite{BSM}, which is proposed to quantify the discrepancy between distinct states within a single MDP, to define a new metric that compares states across two MDPs with different transition probabilities and reward functions. A comparison of BSM and DT-BSM is depicted in Fig.~\ref{fig:bisim}. For clarity, we denote the states in the real MDP as $\{s_1,s_2,\dots\}$ and those corresponding states in the DT MDP as $\{s'_1,s'_2,\dots\}$. Note that the DT MDP and the real MDP share the same state space $\mathcal{S}$, that is, $s_i$ and $s'_i$ represent the identical state, i.e., $s_i=s_i'\in\mathcal{S}$ for all $i$. This equivalence is assumed implicitly in the subsequent discussion.
The proposed DT-BSM, denoted by $\bar{d}$, is defined as a recursive function with respect to the transition probabilities and reward functions of the two MDPs:
\begin{defi} \label{defi:DT-BSM}
The DT-BSM between the real MDP $\langle \mathcal{S}, \mathcal{A}, \mathbb{P}, R, \gamma\rangle$ and the DT MDP $ \langle \mathcal{S}, \mathcal{A}, \mathbb{P}', R', \gamma\rangle$ is given by
\begin{align}\label{eq:dbar}
    \bar{d}\left(s_i, s_j'\right)=\ &\max_{a} \Big\{\big|R(s_i,a)-R^\prime(s_j',a)\big|\notag\\
    &\qquad \ \ +\gamma W_1\big(\mathbb{P}(\cdot|s_i,a), \mathbb{P}^\prime(\cdot|s_j',a) ;\bar{d}\big)\Big\},
\end{align}
where $W_1$ denotes the 1-Wasserstein distance.
\end{defi}
\noindent Detailed definitions and in-depth analysis of the Wasserstein distance and the DT-BSM are provided in Section~\ref{sec:defi}.
\begin{figure}[t]
    \centering
    \includegraphics[width=0.85\linewidth]{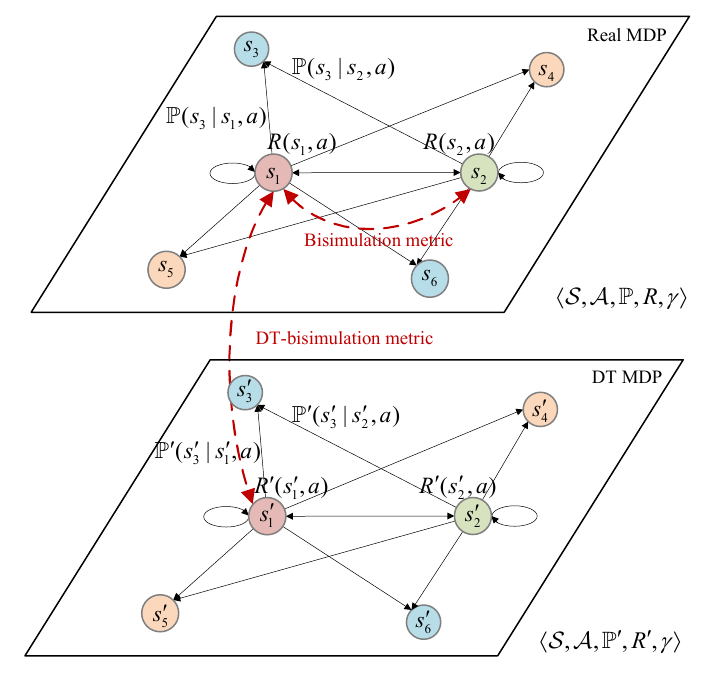}
    \caption{The comparison of BSM and DT-BSM}
    \label{fig:bisim}
\end{figure}

By leveraging the newly defined DT-BSM, we establish upper bounds on the sub-optimality of the transferred policy. Let $\|V_\text{real}^*-V_\text{real}^{\pi}\|$ denote the upper bound of the sub-optimality of policy $\pi$ in the real MDP, formulated by
\begin{equation}\label{opti-define}
    \|V_\text{real}^*-V_\text{real}^{\pi}\| \triangleq \max_{i}\big\{ V_\text{real}^*(s_i)-V_\text{real}^\pi(s_i) \big\},
\end{equation}
and similarly define $\|V_\text{DT}^*-V_\text{DT}^{\pi}\|$ as the upper bound of its sub-optimality in the DT MDP. Then the following theorem holds.
\begin{theorem}\label{the:1}
    For any policy $\pi$ learned in the DT MDP, the sub-optimality of its performance when transferred to the real MDP satisfies
  \begin{align}\label{eq:5}
    & \|V_\textnormal{real}^*-V_\textnormal{real}^{\pi}\| \leq \ \frac2{1-\gamma}\max_{i}\bar{d}(s_i,s_i')+\frac{1+\gamma}{1-\gamma}\|V_\textnormal{DT}^*-V_\textnormal{DT}^{\pi}\|\notag\\
    & \quad  \leq \ \frac{2}{(1-\gamma)^2}\max_{i}d_\textnormal{TV}(s_i,s_i')+\frac{1+\gamma}{1-\gamma}\|V_\textnormal{DT}^*-V_\textnormal{DT}^{\pi}\|,
  \end{align}
where $d_\textnormal{TV}$ is the modified DT-BSM constructed using the total variation distance.
\end{theorem}
\noindent The definition of $d_\text{TV}$ and the proof of Theorem~\ref{the:1} are given in Section~\ref{sec:bound}.

Theorem~\ref{the:1} demonstrates that the sub-optimality of any policy $\pi$ in real-world deployment is bounded by a weighted sum of the DT-BSM and the sub-optimality of $\pi$ within the DT MDP. This additive bound, in contrast to more complex exponential or other functional forms, suggests a promising outlook for the practical application of DT-driven DRL. Specifically, the DT can be independently improved by optimizing the DT-BSM, regardless of the policy employed, while the policy optimization can be safely conducted within the DT environment, eliminating the need for high-risk physical network interactions. Collectively, the proposed bounds theoretically guarantee the deployment performance of DT-trained policy, thereby mitigating the risks associated with deploying unverified policies into physical wireless networks.

\section{Digital Twin Bisimulation Metric}\label{sec:defi}

In this section, we provide a definition of the DT-BSM and develop its key properties for evaluating state discrepancies between the DT MDP and its corresponding real MDP.

\subsection{Definition of DT-BSM}
As discussed earlier, our analysis focuses on the discrepancy between the DT MDP and the real MDP, a discrepancy that is independent of any specific policy $\pi$. Given this, the optimal value function $V^*$, which represents the maximum expected reward achievable across all possible policies in the MDP, serves as a suitable basis for comparison. By taking the difference between $V_\text{real}^*(s)$ and $V_\text{DT}^*(s')$, we have the following inequality, for any $i$ and $j$,
\begin{align}\label{eq:7}
     &\ \big|V_\text{real}^*(s_i)-V_\text{DT}^*(s_j')\big|\notag\\
     =&\ \Big| \max_{a}\Big\{R(s_i,a)+\gamma\sum_{k=1}^{|\mathcal{S}|}\mathbb{P}(s_k|s_i,a)V_\text{real}^*(s_k)\Big\}\notag\\
     &\ - \max_{a}\Big\{R'(s_j',a)+\gamma\sum_{k=1}^{|\mathcal{S}|}\mathbb{P}'(s_k'|s_j',a)V_\text{DT}^*(s_k')\Big\}\Big|\notag\\
    \leq&\ \max_{a}\Big\{\Big| \big( R(s_i,a)+\gamma\sum_{k=1}^{|\mathcal{S}|}\mathbb{P}(s_k|s_i,a)V_\text{real}^*(s_k)\big)\notag\\
     &\qquad\quad - \big(R'(s_j',a)+\gamma\sum_{k=1}^{|\mathcal{S}|}\mathbb{P}'(s_k'|s_j',a)V_\text{DT}^*(s_k')\big)\Big|\Big\}\notag\\
     \leq &\  \max_{a}\Big\{ \Big|R(s_i,a)-R'(s_j',a)\Big|+\gamma\Big|\sum_{k=1}^{|\mathcal{S}|} \mathbb{P}(s_k|s_i,a)V_\text{real}^*(s_k) \notag\\
    &\qquad\quad -\sum_{k=1}^{|\mathcal{S}|} \mathbb{P}'(s_k'|s_j',a)V_\text{DT}^*(s_k')\Big| \Big\}.
\end{align}
In the last inequality in (\ref{eq:7}), the maximization term on the right-hand side consists of two components. The first term clearly represents the discrepancy in the reward functions of the two MDPs. The second term is regarded as a function of $\mathbb{P}(\cdot|s_i,a)$ and $\mathbb{P}'(\cdot|s_j',a)$, so the key challenge in constructing the metric is to identify an appropriate function to characterize the discrepancy between the two $V^*$-weighted transition probabilities.

Drawing on the BSM, we similarly use the 1-Wasserstein distance, also known as Kantorovich--Rubinstein metric \cite{kantorovich1960mathematical}, to quantify the discrepancy. The Wasserstein distance is defined as the minimal cost required to transport mass from one distribution $P$ into another distribution $Q$ \cite{villani2021topics}. Since the comparison is conducted over the DT MDP and the real MDP, we assume both distributions are defined on the same state space $\mathcal{S}$. The transportation plan, denoted by $\boldsymbol{\Lambda}$, is a matrix of dimensions $|\mathcal{S}|\times|\mathcal{S}|$ that specifies the amount of mass transported from each state in $P$ to each state in $Q$. It follows
\begin{equation}\label{matrix}
  \boldsymbol{\Lambda} = \left[\begin{matrix}
\lambda_{1,1} & \lambda_{1,2} & \dots & \lambda_{1,|\mathcal{S}|} \\
\lambda_{2,1} & \lambda_{2,2} & \dots & \lambda_{2,|\mathcal{S}|} \\
\vdots & \vdots & \ddots & \vdots \\
\lambda_{|\mathcal{S}|,1} & \lambda_{|\mathcal{S}|,2} & \dots & \lambda_{|\mathcal{S}|,|\mathcal{S}|}
\end{matrix}\right],
\end{equation}
where $|\mathcal{S}|$ denotes the cardinality of $\mathcal{S}$. Then the Wasserstein distance, i.e., the minimum cost for the transportation plan between $P$ and $Q$ with a cost function $d$, is defined by a linear program as follows.
\begin{align}
W_1(P, Q;d) =\   \min _{\boldsymbol{\Lambda}} &\ \sum_{i,j=1}^{|\mathcal{S}|} \lambda_{i,j} d\left(s_i, s_j'\right), \label{LP0}\\
 \text { subject to} &\ \sum_{j=1}^{|\mathcal{S}|}\lambda_{i,j}=P\left(s_i\right),\ \forall\; i,\label{LP0_cond_1}\\
&\  \sum_{i=1}^{|\mathcal{S}|} \lambda_{i,j}=Q\left(s_j'\right) \label{LP0_cond_2}, \ \forall\; j,\\
&\  \lambda_{i,j} \geq 0,\ \forall\; i,j. \label{LP0_cond_3}
\end{align}
According to the Kantorovich duality \cite{villani2021topics}, this above linear program is equivalent to the following dual linear program:
\begin{align}
       W_1(P, Q;d) =\   \max_{\boldsymbol{\mu},\boldsymbol{\nu}}&\ \sum_{i=1}^{|\mathcal{S}|} \mu_i P(s_i)-\nu_i Q(s_i'),\label{LP}\\
    \text{subject to}&\ \mu_i-\nu_j\leq d(s_i,s_j'),\ \forall\; i,j. \label{LP_cond_1}
\end{align} 
Here, both $\boldsymbol{\mu}$ and $\boldsymbol{\nu}$ are $|\mathcal{S}|$-length vectors. Typically, the cost function $d$ is defined as a distance such that $d(s_i,s_i') = 0$ for identical states $s_i$ and $s_i'$. Considering $P$ and $Q$ are both nonnegative probability distributions, $\mu_i-\nu_i$ needs to be as great as possible to reach the maximum of the objective in (\ref{LP0}). Given that $d(s_i,s_i') = 0$, we have $\mu_i=\nu_i$ for all $i$ as the optimal solution, and thus the objective function reduces to $\max_{\boldsymbol{\mu}}\sum_{i=1}^{|\mathcal{S}|} \mu_i (P(s_i)-Q(s_i'))$ (Kantorovich-Rubinstein theorem \cite{villani2021topics}), as used in BSM. However, the DT-BSM in this paper evaluates the distance across states in two MDPs, where $d(s_i,s_i')$ is not necessarily 0 for the same state in the DT MDP and the real MDP. Therefore, we keep the original form of the dual linear program, as shown in (\ref{LP}) and (\ref{LP_cond_1}).

Following the two linear programs in (\ref{LP0})-(\ref{LP0_cond_3}) and (\ref{LP})-(\ref{LP_cond_1}), we derive two inequalities, respectively. For any $|\mathcal{S}|\times|\mathcal{S}|$ matrix $\boldsymbol{\Lambda}$ with elements $\lambda_{i,j}$ and satisfying conditions (\ref{LP0_cond_1})-(\ref{LP0_cond_3}), we have
\begin{equation}
     W_1(P, Q;d) \leq \sum_{i,j=1}^{|\mathcal{S}|} \lambda_{i,j} d\left(s_i, s_j'\right) \label{leq_ineq}.
\end{equation}
On the other hand, for any pair of $|\mathcal{S}|$-length vectors $\boldsymbol{\mu}$ and $\boldsymbol{\nu}$ satisfying condition (\ref{LP_cond_1}), the following inequality holds
\begin{equation}
    W_1(P, Q;d) \geq \sum_{i=1}^{|\mathcal{S}|} \mu_i P(s_i)-\nu_i Q(s_i') \label{geq_ineq}.
\end{equation}
These two inequalities are pivotal for subsequent derivations.

Now the Wasserstein distance can serve as a function to measure the difference between transition probabilities in the two MDPs, but it still necessitates the cost function $d$ for transportation planning and cost calculation. In most studies, the cost function is given by the Euclidean norm, such as the distance between pixels in the Wasserstein generative adversarial network (WGAN) \cite{pmlr-v70-arjovsky17a}. However, the state space $\mathcal{S}$ in MDP is not a standard Euclidean space, thus the Euclidean norm is not applicable to our problem. To address this, we resort to a recursive definition of the optimal value function, formulated by the following.
\begin{theorem}
    Let $V^{(0)}(s)=0$ be the 0-step optimal value function. Define the $n$-step optimal value function by
    \begin{equation}\label{define:V}
        V^{(n)}(s) =\ \max_{a}\big\{R(s,a)+\gamma\sum_{\tilde{s}\in \mathcal{S}}\mathbb{P}(\tilde{s}|s,a)V^{(n-1)}(\tilde{s})\big\}.
    \end{equation}
    Then $V^{(n)}(s)$ converges to $V^*(s)$ uniformly with $n\rightarrow\infty$.
\end{theorem}
\begin{proof}
    The convergence of $V^{(n)}(s)$ is established via the Banach fixed-point theorem \cite{banach1922operations}. For a detailed proof, see \cite{doi:10.1137/10080484X} and \cite{puterman2014markov}.
\end{proof}
\noindent Similarly, the DT-BSM is defined recursively by the following.
\begin{theorem} \label{the:BSM}
    Let $d_n,n\in\mathbb{N}$ denote a map $\mathcal{S} \times \mathcal{S} \rightarrow [0,\frac{R_\textnormal{max}}{1-\gamma}]$, where $R_\textnormal{max}=\max_{i,j,a} |R(s_i,a)-R^\prime(s_j',a)|$. Let $d_0$ be a constant zero function and define
\begin{align}\label{eq:dn}
    d_{n}\left(s_i, s_j'\right)=&\ \max_{a} \Big\{\big|R(s_i,a)-R^\prime(s_j',a)\big|\notag\\
    &\quad\  +\gamma W_1\big(\mathbb{P}(\cdot|s_i,a), \mathbb{P}^\prime(\cdot|s_j',a) ;d_{n-1}\big)\Big\}.
\end{align}
Then $d_n$ converges to a fixed point $\bar{d}$ uniformly with $n\rightarrow\infty$, which corresponds to the DT-BSM in Definition \ref{defi:DT-BSM}, and the convergence of $d_n$ to $\bar{d}$ satisfies
\begin{equation}
        \bar{d}(s_i, s_j')-d_n(s_i, s_j') \leq \frac{\gamma^n R_\textnormal{max}}{1-\gamma},
\end{equation}
for all $i$ and $j$.
\end{theorem}
\begin{proof}
    The existence of $\bar{d}$ is established through the Knaster-Tarski fixed-point theorem \cite{tarski1955lattice} and the continuity of the recursive function in (\ref{eq:dn}), which is elaborated in the remaining discussions of this subsection. The convergence of $d_n$ to $\bar{d}$ is demonstrated using inequality (\ref{leq_ineq}) and mathematical induction, with details provided in Appendix~B.
\end{proof}

The following proves the existence of $\bar{d}$. Here, we introduce the Knaster-Tarski fixed-point theorem. Let $(\mathcal{X}, \preceq)$ be a partial order, which means certain pairs of elements within the set $\mathcal{X}$ are comparable under the homogeneous relation $\preceq$ \cite{tarski1955lattice}. If this partial order has least upper bounds and greatest lower bounds for its arbitrary subsets, it is called a complete lattice. The Knaster-Tarski fixed-point theorem asserts that for a continuous function on a complete lattice, the iterative application of this function to the least element of the lattice converges to a fixed point $\bar{x}$, which satisfies $\bar{x}=f(\bar{x})$. Formally, the theorem is stated as follows.
\begin{lemma}[Knaster-Tarski Fixed-Point Theorem \cite{tarski1955lattice}]\label{lemma:fixedpoint}  
 Let $(\mathcal{X}, \preceq)$ be a complete lattice and $f: \mathcal{X} \rightarrow \mathcal{X}$ be a continuous function $\Rightarrow$ Then $f$ has a least fixed point, given by 
 \begin{equation}
     \textnormal{fix}(f) = \sqcup_{n \in \mathbb{N}} f^{(n)}(\ubar{x}),
 \end{equation}
 where $\ubar{x}$ is the least element of $\mathcal{X}$, $\sqcup$ denotes the least upper bound, $f^{(n)}(\ubar{x})=f(f^{(n-1)}(\ubar{x}))$, and $f^{(1)}(\ubar{x})=f(\ubar{x})$. Here, the continuity of $f$ is defined such that for any increasing sequence $\{x_n\}$ in $\mathcal{X}$, it satisfies
 \begin{equation}\label{eq:continu}
     f\left(\sqcup_{n \in \mathbb{N}}\left\{x_n\right\}\right)=\sqcup_{n \in \mathbb{N}}\left\{f\left(x_n\right)\right\}.
 \end{equation}
\end{lemma}
Let $\mathcal{M}$ denote the set of all maps, say $d$ and $d'$, that satisfy $\mathcal{S} \times \mathcal{S} \rightarrow [0,\frac{R_\text{max}}{1-\gamma}]$. Equip $\mathcal{M}$ with the usual pointwise ordering: denoting $d\leq d^{\prime}$ if and only if $d(s_i,s_j')\leq d^{\prime}(s_i,s_j')$ for any $i$ and $j$. Then $\mathcal{M}$ forms a complete lattice with the least element $d_0$, i.e., the constant zero function. Regard the recursive definition (\ref{eq:dn}) as a function of $d$ and accordingly define $F: \mathcal{M} \rightarrow \mathcal{M}$ by
\begin{align}
    F\left(s_i, s_j';d\right)=\max_{a}& \Big\{\big|R(s_i,a)-R^\prime(s_j',a)\big|\notag\\
    &+\gamma W_1\big(\mathbb{P}(\cdot|s_i,a), \mathbb{P}^\prime(\cdot|s_j',a);d \big)\Big\}.\label{defin:F}
\end{align}
Utilizing the Knaster-Tarski fixed-point theorem, the existence of $\bar{d}$ can be achieved if the continuity of $F$ on $\mathcal{M}$ holds. 

We first prove the continuity of the second term in $F$. Define $F_{W_1}: \mathcal{M} \rightarrow \mathcal{M}$ by
\begin{equation}\label{defin:Fw}
\begin{aligned}
    F_{W_1}\left(s_i, s_j';d\right)=&W_1\big(\mathbb{P}(\cdot|s_i,a), \mathbb{P}^\prime(\cdot|s_j',a);d \big).
\end{aligned}
\end{equation}
\begin{lemma} \label{lemma:conti:W1}
    $F_{W_1}$ is continuous on $\mathcal{M}$.
\end{lemma}
\begin{proof}
The continuity of $F_{W_1}$ is derived from inequalities (\ref{leq_ineq}) and (\ref{geq_ineq}), detailed in Appendix~A.
\end{proof}
\noindent Armed with Lemma~\ref{lemma:conti:W1}, we are ready to establish the continuity of $F$ as follows.
\begin{lemma}\label{lemma:F}
    $F$ is continuous on $\mathcal{M}$.
\end{lemma}
\begin{proof}
We follow the definition of continuity in (\ref{eq:continu}). Consider an arbitrary increasing sequence $\{\rho_n\}$ on $\mathcal{M}$, for any $i$ and $j$, we have
\begin{align}
&\ F(s_i,s_j';\sqcup_{n\in\mathbb{N}}\{\rho_n\}) \notag\\
\overset{(a)}{=}&\ \max_{a} \Big\{\big|R(s_i,a)-R^\prime(s_j',a)\big|\notag\\
&\qquad\ +\gamma W_1\big(\mathbb{P}(\cdot|s_i,a), \mathbb{P}^\prime(\cdot|s_j',a) ;\sqcup_{n\in\mathbb{N}}\{\rho_n\}\big)\Big\}\notag\\
\overset{(b)}{=}&\ \max_{a} \Big\{\big|R(s_i,a)-R^\prime(s_j',a)\big|\notag\\
&\qquad\ +\gamma \sqcup_{n\in\mathbb{N}} \big\{W_1\big(\mathbb{P}(\cdot|s_i,a), \mathbb{P}^\prime(\cdot|s_j',a) ;\rho_n\big)\big\}\Big\} \notag\\
\overset{(c)}{=}&\ \sqcup_{n\in\mathbb{N}}\Big\{\max_{a}\big\{\big|R(s_i,a)-R^\prime(s_j',a)\big|\notag\\
&\qquad\ +\gamma W_1\big(\mathbb{P}(\cdot|s_i,a), \mathbb{P}^\prime(\cdot|s_j',a) ;\rho_n\big)\big\}\Big\} \notag\\
\overset{(d)}{=}&\ \sqcup_{n\in\mathbb{N}}\{F(s_i,s_j';\rho_n)\}.
\end{align}
Here, steps~$(a)$ and $(d)$ follow from the definition of $F$ given in (\ref{defin:F}), step~$(b)$ applies the continuity of $F_{W_1}$ established in Lemma~\ref{lemma:conti:W1}, and step~$(c)$ relies on the fact that $R(s_i,a)$ and $R^\prime(s_j',a)$ are independent of $n$.
\end{proof}
Now the existence of $\bar{d}$ in Theorem \ref{the:BSM} is established using Lemma~\ref{lemma:fixedpoint} and Lemma~\ref{lemma:F}.

\subsection{Properties of DT-BSM}
As a metric to evaluate the discrepancy between the DT MDP and the real MDP, DT-BSM possesses several unique properties. First, DT-BSM is constructed by modifying the right-hand side of (\ref{eq:7}), thereby preserving the bound on the difference between two optimal value functions as follows.
\begin{corol} \label{bisimineq}
For any $i$ and $j$,
\begin{equation}
    |V_\textnormal{real}^*(s_i)-V_\textnormal{DT}^*(s_j')|\leq \bar{d}(s_i,s_j').
\end{equation}
\end{corol}
\begin{proof}
    This inequality is proved through induction with the recursive definition of $V^*$ in (\ref{define:V}). For the base case, we have
        \begin{align}\label{corol1-1}
            \left|V_\text{real}^{(1)} (s_i)-V_\text{DT}^{(1)} (s_j')\right| =&\  \Big|\max_{a}R(s_i,a) - \max_{a}R^\prime(s_j',a)\Big| \notag\\
            \leq &\  \max_{a} \left|R(s_i,a) - R^\prime(s_j',a)\right|\notag\\
            =&\ d_1(s_i,s_j').
        \end{align}
    By the induction hypothesis, we assume that for any $i$, $j$, and an arbitrary $n$,
    \begin{align}\label{vv_ineq}
        V_\text{real}^{(n)} (s_i)-V_\text{DT}^{(n)} (s_j') \leq&\ |V_\text{real}^{(n)} (s_i)-V_\text{DT}^{(n)} (s_j')|\notag\\
        \leq&\ d_n(s_i,s_j').
    \end{align}
    Then the induction follows
        \begin{align}\label{corol1-3}
           &\ \left|V_\text{real}^{(n+1)} (s_i)-V_\text{DT}^{(n+1)} (s_j')\right| \notag\\
           \overset{(a)}{=} &\ \Big|\max_{a}\Big\{R(s_i,a)+\gamma\sum_{k=1}^{|\mathcal{S}|}\mathbb{P}(s_k|s_i,a)V_\text{real}^{(n)}(s_k)\Big\}\notag\\
           &\ -\max_{a}\Big\{R'(s_j',a)+\gamma\sum_{k=1}^{|\mathcal{S}|}\mathbb{P}'(s_k'|s_j',a)V_\text{DT}^{(n)}(s_k')\Big\}\Big|\notag\\
           \leq&\ \max_{a} \Big\{\Big|\Big(R(s_i,a) +\gamma \sum_{k=1}^{|\mathcal{S}|} \mathbb{P} (s_k|s_i,a) V_\text{real}^{(n)}(s_k)\Big)  \notag\\
           &\qquad\quad\ - \Big(R^\prime (s_j',a) + \gamma \sum_{k=1}^{|\mathcal{S}|}\mathbb{P}^\prime(s_k'|s_j',a) V_\text{DT}^{(n)}(s_k')\Big)\Big|\Big\}\notag\\
           \overset{(b)}{\leq} &\ \max_{a} \Big\{\Big|R(s_i,a) - R^\prime (s_j',a)\Big|  \notag\\
           &\quad +\gamma\Big|\sum_{k=1}^{|\mathcal{S}|} (\mathbb{P} (s_k|s_i,a) V_\text{real}^{(n)}(s_k)- \mathbb{P}^\prime(s_k'|s_j',a) V_\text{DT}^{(n)}(s_k'))\Big|\Big\}\notag\\
           \overset{(c)}{\leq}&\ \max_{a} \!\Big\{\!\!\left|R(s_i,\!a)\!-\! R^\prime (s_j',\!a)\right|\! +\!\gamma W_1\!\big(\mathbb{P}(\cdot|s_i,\!a),\mathbb{P}^\prime(\cdot|s_j',\!a);d_n\!\big)\!\Big\}\notag\\
           \overset{(d)}{=}&\ d_{n+1}(s_i,s_j').
        \end{align}
    Here, steps~$(a)$ and $(d)$ stem from definitions in (\ref{define:V}) and (\ref{eq:dn}), respectively. Step~$(b)$ follows from the triangle inequality. Step~$(c)$ is derived from inequality (\ref{geq_ineq}) and the fact that, according to (\ref{vv_ineq}), $\big( V_\text{real}^{(n)}(s_i) \big)_{i=1}^{|\mathcal{S}|}$ and $\big( V_\text{DT}^{(n)}(s_i') \big)_{i=1}^{|\mathcal{S}|}$ form a pair of $|\mathcal{S}|$-length vectors satisfying condition (\ref{LP_cond_1}) with cost function $d_n$. 
    
   Now from (\ref{corol1-1})-(\ref{corol1-3}), we have $|V_\text{real}^{(n)}(s_i)-V_\text{DT}^{(n)}(s_j')|\leq d_{n}(s_i,s_j'),\ \forall n\in\mathbb{N}$. Taking $n\rightarrow \infty$ yields the desired result.
\end{proof}
The DT-BSM also exhibits a distinct geometric property. For comparison, BSM is a pseudometric which holds the three axioms of a conventional distance metric, that is for all $\{s_i,s_j,s_k\}\in \mathcal{S}$: (1) $d(s_i,s_i)=0$; (2) $d(s_i,s_j)=d(s_j,s_i)$; (3) $d(s_i,s_k)\leq d(s_i,s_j)+d(s_j,s_k)$. Due to the discrepancy in terms of reward and transition probability between the DT MDP and the real MDP, the first two properties are unfortunately no longer held in general by the DT-BSM\footnote{Strictly, the DT bisimulation metric is not a \textit{metric} as it does not satisfy the three axioms like the bisimulation metric. We refer to it as a DT bisimulation metric (DT-BSM) to acknowledge its extension from the bisimulation metric (BSM).}. However, the DT-BSM retains a unique modified form of the triangle inequality, which we refer to as the quadrilateral inequality, as illustrated in Fig. \ref{fig:Inequality} and stated as follows. 
\begin{figure}[t]
    \centering
    \includegraphics[width=0.75\linewidth]{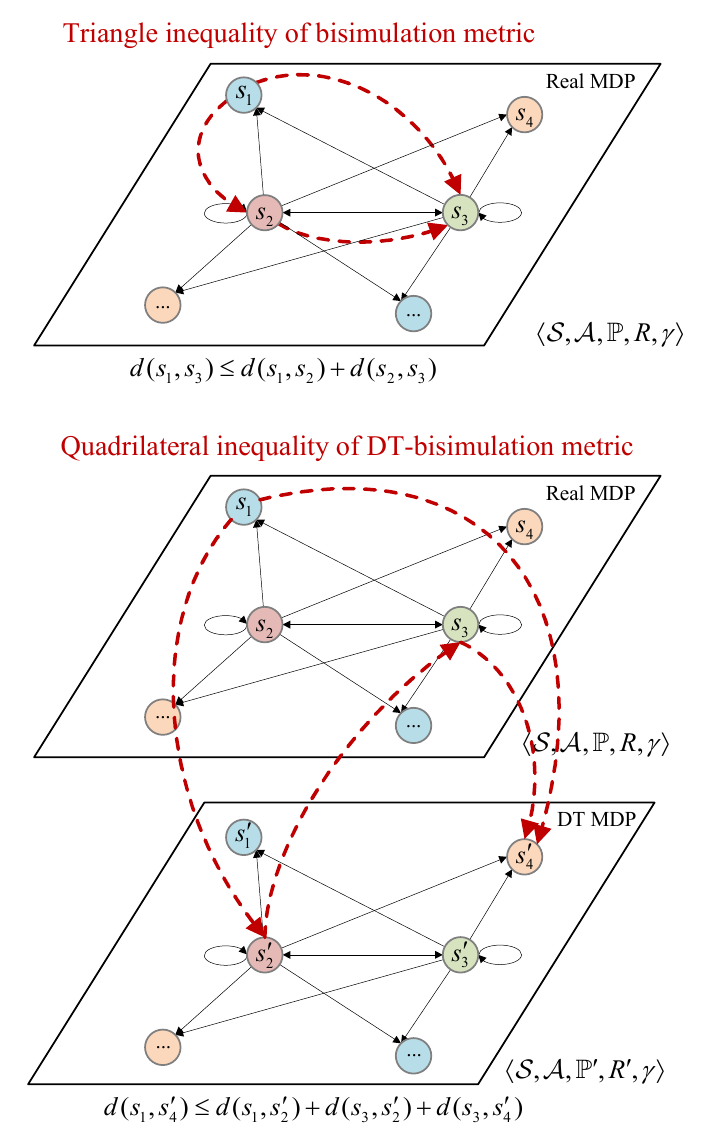}
    \caption{Quadrilateral inequality of DT-BSM.}
    \label{fig:Inequality}
\end{figure}
\begin{corol} \label{inequality}
    For any $\{s_i, s_j',s_k,s_l'\} \in \mathcal{S}$ and $n\in\mathbb{N}$:
    \begin{equation}
        d_n(s_i,s_l')\leq d_n(s_i,s_j')+d_n(s_k,s_j')+d_n(s_k,s_l').
    \end{equation}
\end{corol}
\begin{proof}
    The quadrilateral inequality of DT-BSM is derived from the gluing lemma \cite{villani2009optimal} and induction. The detailed proof is provided in the supplementary material.
\end{proof}
 When the middle two states in the quadrilateral inequality are identical, the expression reduces to a form analogous to the triangle inequality as
\begin{equation}
    d_n(s_i,s_k')\leq d_n(s_i,s_j')+d_n(s_j,s_j')+d_n(s_j,s_k').
\end{equation}

\section{Upper Bound of Sub-optimality in Policy Transfer}\label{sec:bound}

In this section, we present the proofs for the upper bounds of sub-optimality in policy transfer. The upper bounds outlined in Theorem~\ref{the:1} include the one derived from the previously defined DT-BSM $\bar{d}$, and the other based on the computationally efficient total variation distance. Proofs for the two bounds are respectively introduced in the following subsections.

\subsection{Performance Bound Using DT-BSM}
Here, we establish the first inequality in Theorem~\ref{the:1}, which bounds the sub-optimality of transferred policy through the newly defined DT-BSM.
\begin{lemma}\label{lemma:5}
   Given $V_\textnormal{DT}^*$, $V_\textnormal{DT}^\pi$, $V_\textnormal{real}^*$, $V_\textnormal{real}^\pi$, and $\bar{d}$ as defined above,
    \begin{equation}
    \|V_\textnormal{real}^*-V_\textnormal{real}^{\pi}\| \leq \frac2{1-\gamma}\max_{i}\bar{d}(s_i,s_i')+\frac{1+\gamma}{1-\gamma}\|V_\textnormal{DT}^*-V_\textnormal{DT}^{\pi}\|,
\end{equation}
where $\|V_\textnormal{real}^*-V_\textnormal{real}^{\pi}\|$ and $\|V_\textnormal{DT}^*-V_\textnormal{DT}^{\pi}\|$ are the upper bounds of the suboptimality of $\pi$, as defined in (\ref{opti-define}).
\end{lemma}
\begin{proof}
By the triangle inequality, for any state $s_i$, we have
\begin{align}\label{eq33}
    |V_\text{real}^*(s_i)-V_\text{real}^{\pi}(s_i)|\leq&\  |V_\text{real}^*(s_i)-V_\text{DT}^*(s_i')| \notag\\
    +|V_\text{DT}^*(s_i')-&V_\text{DT}^{\pi}(s_i')|  +|V_\text{DT}^{\pi}(s_i')-V_\text{real}^{\pi}(s_i)|.
\end{align}
Within the right-hand side of this inequality, the first summation term $|V_\text{real}^*(s_i)-V_\text{DT}^*(s_i')|$ is upper bounded by $\max_{i}\bar{d}(s_i,s_i')$ according to Corollary \ref{bisimineq}, and $|V_\text{DT}^*(s_i')-V_\text{DT}^{\pi}(s_i')|$ is upper bounded by $\|V_\text{DT}^*-V_\text{DT}^{\pi}\|$ according to its definition in (\ref{opti-define}). For the last term, we have

\begin{align}
    &\ \Big|V_\text{real}^{\pi}(s_i)-V_\text{DT}^\pi(s_i')\Big|\notag\\
    \overset{(a)}{=}&\   \Big| \Big(R(s_i,\pi(s_i))+\gamma\sum_{k=1}^{|\mathcal{S}|} \mathbb{P}(s_k|s_i,\pi(s_i)) V_\text{real}^\pi(s_k))\Big) \notag\\
    &\ -\Big(R'(s_i',\pi(s_i'))+\gamma\sum_{k=1}^{|\mathcal{S}|} \mathbb{P}'(s_k'|s_i',\pi(s_i')) V_\text{DT}^\pi(s_k'))\Big) \Big|\notag\\
     \overset{(b)}{\leq} &\ \Big|R(s_i,\pi(s_i)) - R^\prime (s_i',\pi(s_i'))\Big|   \notag\\
    +\ &\gamma\Big|\sum_{k=1}^{|\mathcal{S}|}\big( \mathbb{P}(s_k|s_i,\pi(s_i)) V_\text{real}^\pi(s_k) -\mathbb{P}^\prime(s_k'|s_i',\pi(s_i')) V_\text{DT}^\pi(s_k')\big) \Big|\notag\\
    \overset{(c)}{\leq} &\  \Big|R(s_i,\pi(s_i)) - R^\prime (s_i',\pi(s_i'))\Big|   \notag\\
   +\ &\gamma\Big|\sum_{k=1}^{|\mathcal{S}|} \big(\mathbb{P}(s_k|s_i,\pi(s_i)) V_\text{real}^*(s_k)- \mathbb{P}^\prime(s_k'|s_i',\pi(s_i')) V_\text{DT}^*(s_k')\big)\Big| \notag\\
    &\ +\gamma \Big|\sum_{k=1}^{|\mathcal{S}|} \mathbb{P}(s_k|s_i,\pi(s_i)) (V_\text{real}^\pi(s_k)-  V_\text{real}^*(s_k) )\Big|\notag\\
    &\ + \gamma \Big|\sum_{k=1}^{|\mathcal{S}|} \mathbb{P}^\prime(s_k'|s_i',\pi(s_i')) (V_\text{DT}^*(s_k')- V_\text{DT}^\pi(s_k'))\Big|  \notag\\
    \leq &\ \max_{a} \Big\{\Big|R(s_i,a) - R^\prime (s_i',a)\Big|   \notag\\
    &\ +\gamma\Big|\sum_{k=1}^{|\mathcal{S}|} (\mathbb{P} (s_k|s_i,a) V_\text{real}^*(s_k) - \mathbb{P}^\prime(s_k|s_i',a) V_\text{DT}^*(s_k))\Big|\Big\} \notag\\
    &\ + \gamma \max_{k} \!\Big|\!V_\text{real}^\pi(s_k) \!-\! V_\text{real}^*(s_k)\Big|\! + \!\gamma \max_{k} \!\Big|\!V_\text{DT}^\pi(s_k') - V_\text{DT}^*(s_k')  \Big| \notag\\
    \overset{(d)}{\leq} &\max_{a} \!\Big\{ \big|\!R(s_i,a) \!-\! R^\prime (s_i',a)\big| \! +\!\gamma W_1\big(\mathbb{P}(\cdot|s_i,a), \mathbb{P^\prime}(\cdot|s_i',a);\bar{d}\big)\!\Big\} \notag\\
    &\ + \gamma \Big\|V_\text{real}^*-V_\text{real}^{\pi}\Big\| + \gamma \Big\|V_\text{DT}^*-V_\text{DT}^{\pi}\Big\|\notag\\
    =&\ \bar{d}(s_i,s_i') + \gamma \|V_\text{real}^*-V_\text{real}^{\pi}\| + \gamma \|V_\text{DT}^*-V_\text{DT}^{\pi}\|.
\end{align}
Here, step~$(a)$ stems from the definition of value function in (\ref{eq:v2}), and steps~$(b)$ and $(c)$ follow from the triangle inequality.
Step~$(d)$ is derived from inequality (\ref{geq_ineq}) and the fact that, according to Corollary \ref{bisimineq}, $\big( V_\text{real}^*(s_i)\big)_{i=1}^{|\mathcal{S}|}$ and $\big( V_\text{DT}^*(s_i) \big)_{i=1}^{|\mathcal{S}|}$ form a pair of $|\mathcal{S}|$-length vectors satisfying condition (\ref{LP_cond_1}) with cost function $\bar{d}$. Combining inequalities on all the three summation terms in (\ref{eq33}), we have
\begin{align}
&\ \|V_\textnormal{real}^*-V_\textnormal{real}^{\pi}\|\notag\\
=&\ \max_{i}\left\{ V_\text{real}^*(s_i)-V_\text{real}^\pi(s_i) \right\}  \notag\\
    \leq&\ \underbrace{\max_{i} \bar{d}(s_i,s_i')}_{\text{1st term}}+ \underbrace{\|V_\text{DT}^*-V_\text{DT}^{\pi}\|}_{\text{2nd term}} \notag\\
   &\  + \underbrace{\max_{i}\bar{d}(s_i,s_i') + \gamma \|V_\text{real}^*-V_\text{real}^{\pi}\| + \gamma \|V_\text{DT}^*-V_\text{DT}^{\pi}\|}_{\text{3rd term}}\notag\\
    \leq&\ 2\max_{i}\bar{d}(s_i,s_i')+(1+\gamma)\|V_\text{DT}^*-V_\text{DT}^{\pi}\|+\gamma \|V_\text{real}^*-V_\text{real}^{\pi}\|.
\end{align}
Rearranging the inequality yields the desired result.
\end{proof}
We now derive an upper bound on the sub-optimality for any policy $\pi$ transferred from the DT to the real environment. However, due to the recursive definition of DT-BSM, $\bar{d}$ must be calculated iteratively, which is computationally exhaustive. Specifically, although only $\bar{d}(s_i,s_i')$ is requisite for the performance bound, calculating the approximated DT-BSM $d_{n}$ at each iteration still necessitates $d_{n-1}(s_i,s_j')$ for distinct $i$ and $j$ to compute the Wasserstein distance. In addition, solving the linear program in (\ref{LP0}) has a complexity of $O(|\mathcal{S}|^2\log|\mathcal{S}|)$ \cite{DOBKIN19801}, so each iteration requires $O(|\mathcal{A}||\mathcal{S}|^4\log|\mathcal{S}|)$ operations. Combining with the convergence of DT-BSM provided in Theorem \ref{the:BSM}, in order to approximate $\bar{d}$ within an error of $\delta$, the number of required operations scales as $O(|\mathcal{A}||\mathcal{S}|^4\log|\mathcal{S}|\frac{\ln(\delta(1-\gamma)/R_\text{max})}{\ln \gamma})$. This substantial complexity renders the DT-BSM computationally infeasible for current DRL-based network optimization tasks in practice, which typically involve large state and action spaces. As exemplified in \cite{9372298}, for a conventional downlink power allocation problem in a multi-cell environment, where $B$ is the number of base stations, $U$ is the number of users, and $F$ is the number of sub-bands. The state space of the MDP is given by $|\mathcal{S}|=(B\times U \times (F+1))$. Considering $L$ power levels for optimization, the action space is $|\mathcal{A}|=(B\times L^F)$. This results in prohibitive computational complexity for the calculation of DT-BSM in emerging heterogeneous and densely deployed 5G networks and its beyond.

\subsection{Performance Bound Using Total Variation Distance}
To address the complexity issue, we adopt the total variation distance, defined as
\begin{equation}
    \textnormal{TV}(P,Q) = \frac{1}{2} \sum_{i=1}^{|\mathcal{S}|} \big|P(s_i)-Q(s_i')\big|,
\end{equation}
to formulate a modified version of DT-BSM, denoted by $d_\text{TV}$. Following the structure of $\bar{d}$, we define $d_\text{TV}$ as follows.
\begin{defi}\label{defi:3}
\begin{align}
    d_\textnormal{TV}(s_i,s_i') =& \ \max_{a} \Big\{\Big|R(s_i,a)-R^\prime(s_i',a)\Big|\notag\\
    &\ \ \ \left.+\frac{\gamma R_\textnormal{max}}{1-\gamma} \textnormal{TV}\left(\mathbb{P}(\cdot|s_i,a), \mathbb{P}^\prime(\cdot|s_i',a) \right)\right\}, \label{defi:3_eq}
\end{align}
where $R_\textnormal{max}=\max_{i,j,a} |R(s_i,a)-R^\prime(s_j',a)|$.
\end{defi}
Then we derive the second inequality in Theorem~\ref{the:1} by demonstrating the relationship between $\bar{d}$ and $d_\textnormal{TV}$.
\begin{lemma}\label{lemma:6}
Given $\bar{d}$ and $d_\textnormal{TV}$ defined respectively in (\ref{eq:dn}) and (\ref{defi:3_eq}),
\begin{equation}
    \max_{i}\bar{d}(s_i,s_i') \leq \frac{1}{1-\gamma}\max_{i}d_\textnormal{TV}(s_i,s_i').
\end{equation}
\end{lemma}
\begin{proof}
\begin{figure}[t]
    \centering
    \includegraphics[width=0.8\linewidth]{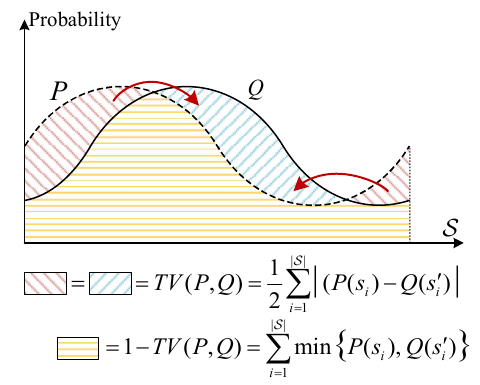}
    \caption{Schematic of the transportation plan.}
    \label{fig:total variation}
\end{figure}
We consider a special transportation plan that moves mass from distribution $P$ to distribution $Q$. This plan preserves the mass shared between $P$ and $Q$, defined as $\min\{P(s_i),Q(s_i')\}$ for each $i$. The remaining mass, where $P(s_i)>Q(s_i')$, is transported to states where $P(s_i)<Q(s_i')$. A schematic example of this plan is shown in Fig.~\ref{fig:total variation}. The total mass to be transported is quantified by the total variation distance, i.e., $\textnormal{TV}(P,Q)$, and the transportation cost with a given cost function $d$ is bounded by $\max_{i,j}{d}(s_i,s_j')$. The shared mass is given by $1-\textnormal{TV}(P,Q)$, with the cost bounded by $\max_{i}{d}(s_i,s_i')$. 

While this plan adheres to the definition of Wasserstein distance, it is not necessarily the optimal one with the minimum transportation cost. Then we have
\begin{align}\label{TVineq}
&\ W_1(P,Q;d)\notag\\
\leq &\ \textnormal{TV}(P,Q)\max_{i,j}d(s_i,s_j') + \big(1-\textnormal{TV}(P,Q)\big)\max_{i}d(s_i,s_i')\notag\\
\overset{(a)}{\leq} &\ \frac{R_\text{max}}{1-\gamma} \textnormal{TV}(P,Q) + \big(1-\textnormal{TV}(P,Q)\big)\max_{i}d(s_i,s_i'),
\end{align}
where step~$(a)$ uses the fact that $d$ is a map $\mathcal{S} \times \mathcal{S} \rightarrow [0,\frac{R_\text{max}}{1-\gamma}]$. Then we obtain
\begin{align}
&\max_{i}\bar{d}(s_i,s_i') \notag\\
\overset{(a)}{=}&\max_{i,a} \!\Big\{ \! \big|R(s_i,\!a)\! - \!R^\prime(s_i',\!a)\big|  \! +\! \gamma W_1\big(\mathbb{P}(\cdot|s_i,\!a), \mathbb{P^\prime}(\cdot|s_i',\!a) ;\bar{d}\big)\! \Big\}  \notag\\
\overset{(b)}{\leq}&\max_{i,a} \Big\{ \!\big|\!R(s_i,\!a)\! - \!R^\prime(s_i',\!a)\!\big|\!  \! +\!  \frac{\gamma R_\text{max}}{1-\gamma}\textnormal{TV}\!\Big(\!\mathbb{P}(\cdot|s_i,\!a), \mathbb{P^\prime}(\cdot|s_i',\!a)\!\Big)\!\notag\\
&\qquad+\gamma\Big(1-\textnormal{TV}\big(\mathbb{P}(\cdot|s_i,\!a), \mathbb{P^\prime}(\cdot|s_i',\!a) \big)\Big)\max_{j}\bar{d}(s_j,s_j')\! \Big\} \notag\\
\leq& \max_{i,a} \!\Big\{ \! \big|\!R(s_i,\!a) \! - \!R^\prime(s_i',\!a)\!\big| \! +\! \frac{\gamma R_\text{max}}{1-\gamma}\textnormal{TV}\big(\mathbb{P}(\cdot|s_i,\!a), \mathbb{P^\prime}(\cdot|s_i',\!a)\!\big)\!\! \Big\} \notag\\
&+\gamma \max_{i,a} \!\Big\{ \! \Big(1-\textnormal{TV}\big(\mathbb{P}(\cdot|s_i,\!a), \mathbb{P^\prime}(\cdot|s_i',\!a) \big)\Big)\max_{j}\bar{d}(s_j,s_j')\! \Big\} \notag\\
\overset{(c)}{\leq}& \max_{i,a} \!\Big\{ \! \big|\!R(s_i,\!a)\! - \!R^\prime(s_i',\!a)\!\big|  \! +\!  \frac{\gamma R_\text{max}}{1-\gamma}\textnormal{TV}\big(\mathbb{P}(\cdot|s_i,\!a), \mathbb{P^\prime}(\cdot|s_i',\!a )\!\big)\!\! \Big\}  \notag\\
&+ \gamma\max_{i}\bar{d}(s_i,s_i').
\end{align}

Here, step~$(a)$ follows from the definition of $\bar{d}$ in (\ref{eq:dbar}), step~$(b)$ is derived from (\ref{TVineq}), and step~$(c)$ is obtained from the non-negativity of the total variation distance.
Rearranging both sides of the above inequality, we have
\begin{align}
\max_{i}\bar{d}(s_i,s_i') \leq&\ \frac{1}{1-\gamma}\max_{i,a} \Big\{\Big|R(s_i,a)-R^\prime(s_i',a)\Big|\notag\\
&\ + \frac{\gamma R_\text{max}}{1-\gamma}\textnormal{TV}\big(\mathbb{P}(\cdot|s_i,a), \mathbb{P^\prime}(\cdot|s_i',a) \big)\Big\}\notag\\
=&\ \frac{1}{1-\gamma}\max_{i}d_\text{TV}(s_i,s_i'),
\end{align}
which established the required inequality.
\end{proof}

Now that Theorem~\ref{the:1} readily follows from Lemma~\ref{lemma:5} and Lemma~\ref{lemma:6}. Compared to the $O(|\mathcal{A}||\mathcal{S}|^4\log|\mathcal{S}|\frac{\ln(\delta(1-\gamma)/R_\text{max})}{\ln \gamma})$ operations required to approximate $\max_{i}\bar{d}(s_i,s_i')$, the total variation distance method requires only $O(|\mathcal{A}||\mathcal{S}|^2)$ operations to compute $\max_{i}d_\text{TV}(s_i,s_i')$, while consistently guaranteeing the performance bound.

\section{Empirical DT-BSM with the Total Variation Distance}\label{sec:emp}
In the previous section, we established a performance bound using $\bar{d}$ and introduced a computationally efficient alternative, $d_\text{TV}$, based on the total variation distance. While $d_\text{TV}$ significantly reduces the computational complexity of calculating DT-BSM, it is still challenging in practice to acquire accurate transition probabilities in both real and DT environments.

A promising approach is statistical sampling. Suppose $P$ and $Q$ are approximated by the empirical distributions $\hat{P}$ and $\hat{Q}$, respectively. To be specific, we collect $K$ independent samples $\{X_1, X_2, \ldots, X_K\}$ from $P$ and define $\hat{P}(x_i)=\frac{1}{K} \sum_{k=1}^K \delta_{X_k}(x_i)$, where $\delta$ denotes the Dirac measure at $X_k$ such that $\delta_{X_k}(x_i)=1$ if $x_i = X_k$ and $0$ otherwise. Similarly, we define $\hat{Q}$. By the strong law of large numbers (SLLN), $\hat{P}$ and $\hat{Q}$ converge almost surely to $P$ and $Q$, respectively. Based on these empirical distributions, we define the empirical total variation distance, $\hat{\textnormal{TV}}(P,Q)$, and the empirical DT-BSM on total variation distance, $\hat{d}_\text{TV}(s_i,s_i')$ as follows.
\begin{defi}\label{defi:4}
    Given empirical distributions $\hat{P}$ and $\hat{Q}$ approximated through $K$ samples, define the empirical total variance by
\begin{equation}
    \hat{\textnormal{TV}}(P,Q) = \frac{1}{2} \sum_{i=1}^{|\mathcal{S}|} |\hat{P}(s_i)-\hat{Q}(s_i')|,
\end{equation}
    and define the empirical DT-BSM on the total variation distance by
\begin{align}\label{defi:4_eq}
    \hat{d}_\textnormal{TV}(s_i,s_i') =  &\max_{a} \Big\{\Big|R(s_i,a)-R^\prime(s_i',a)\Big|\notag\\
    &\ \left.+\frac{\gamma R_\textnormal{max}}{1-\gamma} \hat{\textnormal{TV}}\left(\mathbb{P}(\cdot|s_i,a), \mathbb{P}^\prime(\cdot|s_i',a) \right)\right\}.
\end{align}
\end{defi}
\noindent Then we have the following lemma on the convergence of $\hat{d}_\text{TV}$.
\begin{lemma}\label{lemma:emprical}
$\hat{d}_\textnormal{TV}$ converges to $d_\textnormal{TV}$ almost surely as $K$ gets large.
\end{lemma}
\begin{proof}
By the SLLN, we can choose a sufficiently large $K$ such that for all $i$,
    \begin{equation}\label{lemma6:1}
        \max\Big\{\big|\hat{P}(s_i)-P(s_i)\big|,\big|\hat{Q}(s_i')-Q(s_i')\big|\Big\}\leq\frac{\epsilon(1-\gamma)}{\gamma R_\text{max}|\mathcal{S}|},
    \end{equation}
where $\epsilon>0$. Then
    \begin{align}
        &\ \big|\textnormal{TV}(P,Q)-\hat{\textnormal{TV}}(P,Q)\big|\notag\\
        = &\ \Big|\frac{1}{2} \sum_{i=1}^{|\mathcal{S}|} \big|P(s_i)-Q(s_i')\big|-\frac{1}{2} \sum_{i=1}^{|\mathcal{S}|} \big|(\hat{P}(s_i)-\hat{Q}(s_i')\big|\Big|\notag\\
        \leq &\ \frac{1}{2} \sum_{i=1}^{|\mathcal{S}|} \Big|\big|P(s_i)-Q(s_i')\big|- \big|(\hat{P}(s_i)-\hat{Q}(s_i')\big|\Big|\notag\\
        \leq &\ \frac{1}{2} \sum_{i=1}^{|\mathcal{S}|} \Big|\big(P(s_i)-Q(s_i')\big)- \big((\hat{P}(s_i)-\hat{Q}(s_i')\big)\Big|\notag\\
        \leq &\ \frac{1}{2} \sum_{i=1}^{|\mathcal{S}|} \Big|\big|P(s_i)-\hat{P}(s_i)\big| + \big|(\hat{Q}(s_i')-Q(s_i')\big|\Big|\notag\\
       \overset{(a)}{\leq} &\ \frac{|\mathcal{S}|}{2} \cdot \Big|\frac{\epsilon(1-\gamma)}{\gamma R_\text{max}|\mathcal{S}|} +  \frac{\epsilon(1-\gamma)}{\gamma R_\text{max}|\mathcal{S}|}\Big| =\frac{\epsilon(1-\gamma)}{\gamma R_\text{max}},\label{EmpiriTVineq}
    \end{align}
where step~$(a)$ follows from (\ref{lemma6:1}). Next,
\begin{align}
        &\ \big|d_\text{TV}\left(s_i, s_i'\right) - \hat{d}_\text{TV}\left(s_i, s_i'\right)\big|\notag\\
        \overset{(a)}{=} &\ \Big|\!\max_{a}\! \Big\{\!\big|\!R(s_i,a)\!-\!R^\prime(s_i',a)\!\big|\!\! +\!\!\frac{\gamma\! R_\text{max}}{1-\gamma} \textnormal{TV}\!\big(\mathbb{P}(\cdot|s_i,a), \mathbb{P}^\prime(\cdot|s_i',a) \!\big)\!\!\Big\}\notag\\
        &\ -\max_{a} \Big\{\big|R(s_i,a)-R^\prime(s_i',a)\big|\notag\\
        &\ \qquad\quad\ \ +\frac{\gamma R_\text{max}}{1-\gamma} \hat{\textnormal{TV}}\big(\mathbb{P}(\cdot|s_i,a), \mathbb{P}^\prime(\cdot|s_i',a) \big)\Big\}\Big|\notag\\
        \leq &\ \max_{a}\Big\{\Big|\Big(\big|R(s_i,a)-R^\prime(s_i',a)\big|\notag\\
        &\qquad\quad\ +\frac{\gamma R_\text{max}}{1-\gamma} \textnormal{TV}\big(\mathbb{P}(\cdot|s_i,a), \mathbb{P}^\prime(\cdot|s_i',a) \big)\Big)\notag\\
        &\qquad\quad\ -\Big(\big|R(s_i,a)-R^\prime(s_i',a)\big|\notag\\
        &\qquad\quad\ +\frac{\gamma R_\text{max}}{1-\gamma} \hat{\textnormal{TV}}\big(\mathbb{P}(\cdot|s_i,a), \mathbb{P}^\prime(\cdot|s_i',a) \big)\Big)\Big|\Big\}\notag\\
         =&\  \frac{\gamma R_\text{max}}{1-\gamma} \max_{a} \Big\{\Big|\textnormal{TV}\big(\mathbb{P}(\cdot|s_i,a), \mathbb{P}^\prime(\cdot|s_i',a)\big)\notag\\
        &\qquad\qquad\qquad-\hat{\textnormal{TV}}\big(\mathbb{P}(\cdot|s_i,a), \mathbb{P}^\prime(\cdot|s_i',a)\big)\Big|\Big\}\notag\\
        \overset{(b)}{\leq} &\  \frac{\gamma R_\text{max}}{1-\gamma}\cdot \frac{\epsilon(1-\gamma)}{\gamma R_\text{max}} = \epsilon.
\end{align}
Here, step~$(a)$ follows from the definitions in (\ref{defi:3_eq}) and (\ref{defi:4_eq}), and step~$(b)$ uses the result in (\ref{EmpiriTVineq}).
\end{proof}
The required number of samples for approximation can be derived through the Hoeffding's inequality \cite{hoeffding1994probability}. Without loss of generality, consider a transition start at state $s_i$ with action $a$. By recording the next state upon taking action $a$ at state $s$ for a total of $K$ times, we obtain $K$ samples. For each potential next state $s_k\in\mathcal{S}$, this process can be modeled as Bernoulli trials with a success probability given by $\mathbb{P}(s_k|s_i,a)$. The empirical probability $\hat{\mathbb{P}}(s_k|s_i,a)$ is defined as the ratio of the number of successful outcomes (i.e., the transitions into $s_k$) to the totally $K$ trials. According to the Hoeffding's inequality, we have
\begin{equation}
    \operatorname{Pr}\Big\{\big|\hat{\mathbb{P}}(s_k|s_i,a)-\mathbb{P}(s_k|s_i,a)\big|\geq\epsilon\Big\} \leq 2e^{-2K\epsilon^2}
\end{equation}
for each $s_k$, where $\operatorname{Pr}\{\cdot\}$ represents the probabiltiy of an event. To achieve a confidence level of $1-\alpha$ for $\alpha\in[0,1)$, we require
\begin{align}
    2e^{-2K\epsilon^2}\leq\alpha\ \Longrightarrow\ K\geq \frac{-\ln(\alpha/2)}{2\epsilon^2}.\label{hoeffding}
\end{align}
This result applies to the transition probabilities in both DT MDP and the real MDP. Combining (\ref{hoeffding}) with (\ref{lemma6:1}) in Lemma~\ref{lemma:emprical}, we determine the required sample size for the empirical $d_\text{TV}$.
\begin{theorem}
    For each state-action pair in both DT MDP and the real MDP, collect $K$ samples such that
    \begin{equation}
        K\geq -\ln(\alpha/2) \frac{\gamma^2R_\textnormal{max}^2|\mathcal{S}|^2}{2\epsilon^2(1-\gamma)^2}.
    \end{equation}
    Then, with confidence level $1-\alpha$, we have 
    \begin{equation}
        |d_\textnormal{TV}(s_i,s_i')-\hat{d}_\textnormal{TV}(s_i,s_i')| \leq \epsilon,
    \end{equation}
    where $d_\textnormal{TV}(s_i,s_i')$ and $\hat{d}_\textnormal{TV}$ represent the theoretical and empirical DT-BSM based on the total variation distance, as defined in (\ref{defi:3_eq}) and (\ref{defi:4_eq}), respectively.
\end{theorem}
This result enables the calculation of a sufficiently accurate DT-BSM through sampling in both DT and real environments. Notably, the sampling process is conducted for each state-action pair, independent of the policy being followed. As a result, sampling in the real-world network can be carried out prior to the deployment of DRL policies, ensuring the reliable performance of physical networks.

\section{Numerical Experiments}\label{sec:exp}

In this section, we validate the derived performance bounds through a specific network optimization task, that is, admission control in sliced wireless networks \cite{van2019optimal}.

We consider a typical 5G wireless network with three distinct slices, each catering to one of the typical services, i.e., enhanced mobile broadband (eMBB), massive machine-type communications (mMTC), and ultra-reliable and low latency communications (URLLC) \cite{3GPP.21.915}. The network involves three resource types, including radio, computing, and storage, with each slice having unique resource requirements. For example, eMBB services primarily rely on radio resources, while URLLC services have a higher demand for computing resources.

Service requests arise from end users continuously in the wireless network, where the admission control policy evaluates the current request and available resources to determine whether some requests can be admitted and their priority. Requests not admitted remain in the queue until either accepted or expired due to a timeout. The request arrival process follows a Poisson distribution, and the service time (network resource occupation time) follows an exponential distribution \cite{van2019optimal}.

In terms of the associated MDP, the state is defined by the number of requests in the queue and the number of ongoing services in the network. The action is a 3-dimensional vector, representing the number of service requests to be admitted for each slice. The transition probability is influenced by the mean arrival rate of the Poisson distribution and the mean service time of the exponential distribution. The reward is defined as the expected profit from accepting service requests. The proximal policy optimization (PPO) algorithm is utilized to develop the optimal admission control policy. 

In the experiment, we establish both a real environment with accurate parameter setting and a DT environment that includes potential estimation errors in the arrival/service processes (transition probability) and the profit associated with each request (reward). For each test, the DRL agent is trained in the DT environment and then deployed in the real network to validate its performance. Since the optimal policy is not directly attainable, the sub-optimality in deployment is assessed by the difference in long-term average reward between the transferred policy and the policy directly trained in the real environment. In addition, according to Theorem~\ref{the:1},
\begin{align}
    &\|V_\text{real}^*-V_\text{real}^{\pi}\| \leq\  \frac{2}{(1-\gamma)^2}\max_{i,a} \Big\{\Big|R(s_i,a)-R^\prime(s_i',a)\Big| \notag\\
    &+\frac{\gamma R_\text{max}}{1-\gamma} \textnormal{TV}\left(\mathbb{P}(\cdot|s_i,a), \mathbb{P}^\prime(\cdot|s_i',a) \right)\Big\}+\frac{1+\gamma}{1-\gamma}\|V_\text{DT}^*-V_\text{DT}^{\pi}\|,
\end{align}
where $\frac{1+\gamma}{1-\gamma}\|V_\text{DT}^*-V_\text{DT}^{\pi}\|$ reflects the sub-optimality caused by insufficient training in the DT environment. We conduct thorough DRL training in the DT environment to minimize this influence and denote it as $\Delta$ in the following.

\begin{figure}[t]
    \centering
    \includegraphics[width=0.822\linewidth]{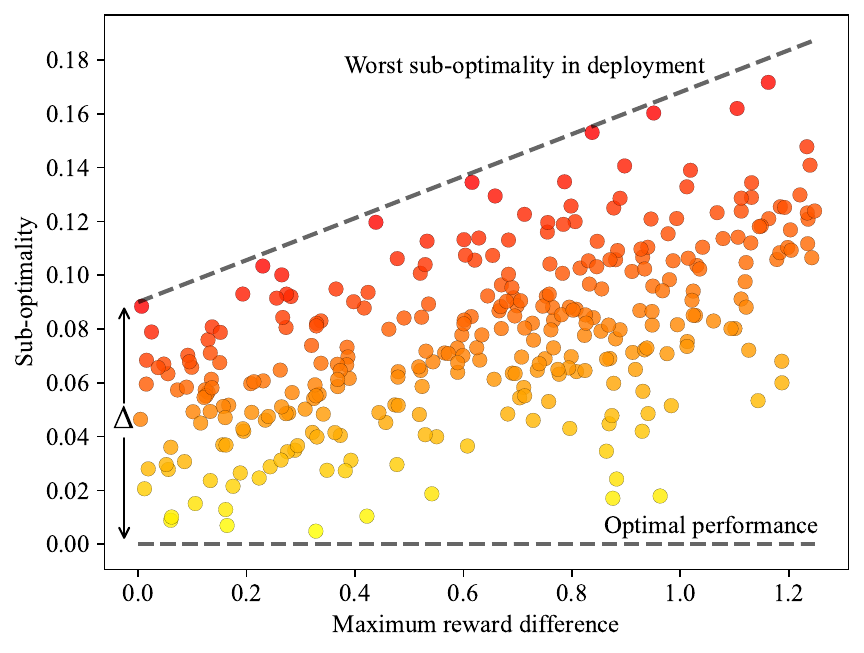}
    \caption{Sub-optimality of transferred policy versus reward difference}
    \label{fig:rwd}
\end{figure}

First, we maintain consistent arrival and service processes in both environments, focusing on the error in profit estimation which introduces a discrepancy in the reward function of the two MDPs. By aggregating the constants in the inequality, the performance bound becomes
\begin{equation}
    \|V_\text{real}^*-V_\text{real}^{\pi}\|\leq \beta_1 \max_{i,a} \big|R(s_i,a)-R^\prime(s_i',a)\big| + \Delta,
\end{equation}
where $\beta_1$ is a constant. We conduct hundreds of training-deployment tests with varying reward function errors in DT and compare the sub-optimality of the performance of transferred policy in the real environment. As shown in Fig.~\ref{fig:rwd}, the worst sub-optimality of the deployment performance exhibits an approximately linear relationship with the reward function discrepancy, which aligns with our theoretical analysis.

\begin{figure}[t]
    \centering
    \includegraphics[width=0.822\linewidth]{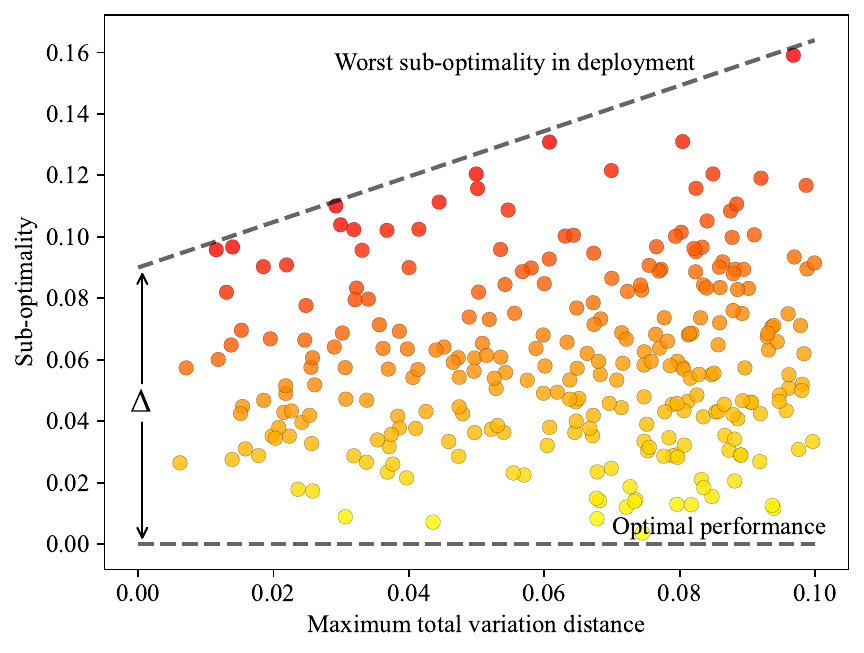}
    \caption{Sub-optimality of transferred policy versus total variation distance of transition probabilities}
    \label{fig:Prob}
\end{figure}

Next, we consider the error in the mean estimation of the arrival and service processes, which leads to a discrepancy in the transition probabilities. Detached from the impact of rewards, the performance bound is given by
\begin{equation}
    \|V_\text{real}^*-V_\text{real}^{\pi}\|\leq \beta_2 \max_{i,a} \textnormal{TV}\big(\mathbb{P}(\cdot|s_i,a), \mathbb{P}^\prime(\cdot|s_i',a)\big) + \Delta,
\end{equation}
where $\beta_2$ is a constant. Due to the large state space, we select a few representative states from both MDPs and compute the maximum total variation distance. Similar experiments are conducted for the DT environments with varying transition probabilities, and the results are presented in Fig.~\ref{fig:Prob}. Consistently, the worst sub-optimality of the performance of the deployed policy shows a nearly linear relationship with the total variation distance in the transition probabilities.

The numerical results demonstrate that, for a DT-driven DRL-based network optimization task, the worst-case deployment performance of a fully DT-trained policy approaches the optimal almost linearly as the discrepancy between the DT MDP and the real MDP decreases, both in terms of transition probabilities and reward functions. Therefore, the proposed DT-BSM provides a quantitative measure to evaluate the DT's quality for DRL training, thus ensuring reliable deployment performance of DT-trained policies.

\section{Conslusion}\label{sec:con}

In this paper, we propose a novel metric, termed DT-BSM, for DT-driven DRL in wireless network optimization. This metric evaluates the DT's quality for reliable DRL training by measuring the discrepancy between the DT MDP and the real MDP. We prove that for any policy learned within the DT MDP (regardless of whether it is DRL-based), the sub-optimality of its deployment performance in the real MDP is bounded by a weighted sum of the DT-BSM and its sub-optimality in the DT MDP. To enhance applicability in large-scale wireless networks, the modified DT-BSM is introduced using the total variation distance, which exponentially reduces the computational complexity from $O(|\mathcal{A}||\mathcal{S}|^4\log|\mathcal{S}|\frac{\ln(\delta(1-\gamma)/R_\text{max})}{\ln \gamma})$ to $O(|\mathcal{A}||\mathcal{S}|^2)$. Furthermore, for real-world scenarios where accurate transition probabilities of the MDP are hardly accessible, we propose an empirical DT-BSM derived through statistical sampling. We prove the convergence from the empirical DT-BSM to the theoretical one and quantitatively establish the relationship between the required sample size and the target level of approximation accuracy. Numerical experiments on a typical admission control task corroborate the theoretical findings. To the best of our knowledge, this is the first metric to provide provable and calculable performance bounds for DT-driven DRL.

\appendices

\section*{Appendix~A\\The continuity of $F_{W_1}$}
We follow the definition of continuity defined in Lemma~\ref{lemma:fixedpoint}. Regard $F_{W_1}\left(s_i, s_j';d\right)$ as a function of $d$. Without loss of generality, fix probability distributions $\{\mathbb{P}(\cdot|s_i,a), \mathbb{P}^\prime(\cdot|s_j',a)\}$ as $\{P,Q\}$ for brevity, and let $\rho\leq \rho^\prime$, $\{\rho,\rho^\prime\}\in\mathcal{M} $. Consider an arbitrary optimal solution $\{\boldsymbol{\mu},\boldsymbol{\nu}\}$ for $W_1(P, Q;\rho)$ in the dual linear program in (\ref{LP}), we have
\begin{equation}
 \mu_i-\nu_j\leq \rho(s_i,s_j')\leq\rho'(s_i,s_j'),\ \forall\;i,j,
\end{equation}
 which is derived from the condition in (\ref{LP_cond_1}) and the pointwise ordering in $\mathcal{M}$.
Here, $\{\boldsymbol{\mu},\boldsymbol{\nu}\}$ is a pair of vectors satisfying condition (\ref{LP_cond_1}) for $W_1(P, Q;\rho')$, and according to inequality (\ref{geq_ineq}),
\begin{align}\label{lemma2ieq1}
\ W_1(P, Q;\rho)= &\ \sum_{i=1}^{|\mathcal{S}|} \mu_i P(s_i)-\nu_i Q(s_i')\notag\\
            \leq &\ W_1(P, Q;\rho^\prime).
\end{align}
By such a monotonicity, we have $\sqcup_{n \in \mathbb{N}} \{W_1(P,Q;\rho_n)\}\leq W_1(P,Q;\sqcup_{n \in \mathbb{N}} \{\rho_n\})$ for any increasing sequence $\{\rho_n\}$ on $\mathcal{M}$.  

We use the original linear program for the other side. Let $\boldsymbol{\Lambda}^n$ denote the optimal solution in (\ref{LP0}) for $W_1(P,Q;\rho_n)$ with elements $\lambda^n_{i,j}$, which inherently satisfies conditions (\ref{LP0_cond_1})-(\ref{LP0_cond_3}) for $W_1(P, Q;\sqcup_{n \in \mathbb{N}} \{\rho_n\})$. Define $\epsilon^n_{i,j} = \sqcup_{n \in \mathbb{N}} \{\rho_n\}(s_i,s_j')-\rho_n(s_i,s_j')$, then $\epsilon^n_{i,j} \geq 0 $ and $\lim_{n\rightarrow \infty}\epsilon^n_{i,j} = 0$ due to the monotonicity of the increasing sequence. Then, we have
\begin{align}\label{lemma2ieq2}
&\ W_1(P,Q;\sqcup_{n \in \mathbb{N}}\{\rho_n\})\notag\\
\overset{(a)}{\leq}&\ \sum_{i,j=1}^{|\mathcal{S}|}\lambda^n_{i,j} \cdot \sqcup_{n \in \mathbb{N}}\{\rho_n\}(s_i,s_j') \notag\\
\overset{(b)}{=}&\  \sum_{i,j=1}^{|\mathcal{S}|} \lambda^n_{i,j} \rho_n(s_i,s_j')+\sum_{i,j=1}^{|\mathcal{S}|} \lambda^n_{i,j} \epsilon^{n}_{i,j}\notag\\
\overset{(c)}{=}&\ W_1(P,Q;\rho_n) + \sum_{i,j=1}^{|\mathcal{S}|} \lambda^n_{i,j} \epsilon^{n}_{i,j}\notag\\
\overset{(d)}{\leq}&\ \sqcup_{n \in \mathbb{N}} \{W_1(P,Q;\rho_n)\} +\sum_{i,j=1}^{|\mathcal{S}|} \min\{P(s_i),Q(s_j')\} \mathbf{\epsilon}^{n}_{i,j}.
\end{align}
Here, step~$(a)$ utilizes inequality (\ref{leq_ineq}), step~$(b)$ applies the definition of $\epsilon^n_{i,j}$, step~$(c)$ applies the definition of $W_1$ in (\ref{LP0}), and the second term in step~$(d)$ is based on conditions (\ref{LP0_cond_1}) and (\ref{LP0_cond_2}).
Taking $n\rightarrow \infty$, we have $\sqcup_{n \in \mathbb{N}} \{W_1(P,Q;\rho_n)\}\geq W_1(P,Q;\sqcup_{n \in \mathbb{N}} \{\rho_n\})$. Following from the two inequalities, we have $\sqcup_{n \in \mathbb{N}} \{W_1(P,Q;\rho_n)\}= W_1(P,Q;\sqcup_{n \in \mathbb{N}} \{\rho_n\})$, and thus for any $i$ and $j$,
\begin{align}
&\ F_{W_1}\big(s_i, s_j';\sqcup_{n \in \mathbb{N}}\{\rho_n\}\big) \notag\\
=&\ W_1\big(\mathbb{P}(\cdot|s_i,a), \mathbb{P}^\prime(\cdot|s_j',a);\sqcup_{n \in \mathbb{N}}\{\rho_n\}\big)\notag\\
=&\ \sqcup_{n \in \mathbb{N}} \big\{W_1(\mathbb{P}(\cdot|s_i,a), \mathbb{P}^\prime(\cdot|s_j',a);\rho_n)\big\}\notag\\
=&\ \sqcup_{n \in \mathbb{N}}\big\{ F_{W_1}\left(s_i, s_j';\rho_n\right)\big\}.
\end{align}
Now that the continuity of $F_{W_1}$ on $\mathcal{M}$ is established.

\section*{Appendix~B\\The Convergence of DT-BSM}
Due to the continuity of the recursive function in (\ref{eq:dn}) and using the induction starting from $d_0\leq d_1$, $\{d_n\}$ forms an increasing sequence on $\mathcal{M}$. Given that $\bar{d}=\sqcup_{n \in \mathbb{N}} F^{(n)}(d_0)$, we have $\bar{d}\geq d_n$ for any $n$. Then we begin with a simple inequality for the Wasserstein distance before proving the convergence of DT-BSM. Let $\boldsymbol{\Lambda}$ with elements $\lambda_{i,j}$ denote the optimal solution for $W_1(P,Q;\bar{d})$, then for any $d_n$
\begin{align}\label{appenBineq}
        &\ W_1\big(P,Q;\bar{d}\big)\notag\\
        = &\  \sum_{i,j}^{|\mathcal{S}|} \lambda_{i,j} \bar{d}(s_i, s_j')\notag\\
        =&\  \sum_{i,j}^{|\mathcal{S}|} \lambda_{i,j} \big(\bar{d}(s_i, s_j')-d_n(s_i, s_j')+d_n(s_i, s_j') \big) \notag\\
        \leq&\  \max_{i,j} \big\{\bar{d}(s_i, s_j')-d_n(s_i, s_j')\big\}+ W_1\big(P,Q;d_n\big).
\end{align}
In the last inequality, the first summation term is based on conditions (\ref{LP0_cond_1})-(\ref{LP0_cond_3}) and $\bar{d}\geq d_n$, while the second term follows from inequality (\ref{leq_ineq}) and the fact that $\boldsymbol{\Lambda}$ satisfies conditions (\ref{LP0_cond_1})-(\ref{LP0_cond_3}) for $W_1\left(P,Q ;d_n\right)$. 

Now we employ the mathematical induction. For the base case, for any $i$ and $j$, we have
\begin{align}\label{AppenB:eq1}
    \ &\bar{d}(s_i, s_j')-d_1(s_i, s_j') \notag\\
    =\ & \max_{a} \!\Big\{\!\big| R(s_i,a)\!-\!R'(s_j',a) \big|\! +\!\gamma W_1\big(\mathbb{P}(\cdot|s_i,a), \mathbb{P}'(\cdot|s_j',a) ;\bar{d}\big)\!\Big\}\notag\\
    \ &-\max_{a} \Big\{\big|R(s_i,a)-R'(s_j',a)\big|\Big\}\notag\\
    \leq\ & \max_{a} \Big\{\big|R(s_i,a)-R'(s_j',a)\big|\Big\}\notag\\
    \ & + \gamma \max_{a} \Big\{ W_1\big(\mathbb{P}(\cdot|s_i,a), \mathbb{P}'(\cdot|s_j',a) ;\bar{d}\big)\Big\}\notag\\
    \ &-\max_{a} \Big\{\big|R(s_i,a)-R'(s_j',a)\big|\Big\}\notag\\
    =\  & \gamma \max_{a} \Big\{W_1\big(\mathbb{P}(\cdot|s_i,a), \mathbb{P}'(\cdot|s_j',a) ;\bar{d}\big)\Big\}\notag\\
    \leq\  & \gamma \max_{i,j} \{\bar{d}(s_i,s_j')\} = \frac{\gamma R_\text{max}}{1-\gamma}.
\end{align}
By the induction hypothesis, we assume that for any $i$, $j$, and an arbitrary $n$,
\begin{align}\label{AppenB:eq2}
    &\bar{d}(s_i, s_j')-d_n(s_i, s_j') \leq \frac{\gamma^n R_\text{max}}{1-\gamma}.
\end{align}
Then for any $i$ and $j$, we have

\begin{align}\label{AppenB:eq3}
    &\ \bar{d}(s_i, s_j')-d_{n+1}(s_i, s_j') \notag\\
    \overset{(a)}{=}\ & \max_{a}\! \Big\{ \!\big|\!R(s_i,\!a)\!-\!R'(s_j',\!a)\!\big| \!+\!\gamma W_1\big(\mathbb{P}(\cdot|s_i,\!a), \mathbb{P}'(\cdot|s_j',\!a) ;\bar{d}\big) \!\Big\}\notag\\
    \ &-\max_{a}\! \Big\{ \!\big|\!R(s_i,\!a)\!-\!R'(s_j',\!a)\!\big| \!+\!\gamma W_1\!\big(\mathbb{P}(\!\cdot|s_i,\!a), \mathbb{P}'\!(\!\cdot|s_j',\!a) ;\!d_n\!\big) \!\Big\}\notag\\
    \overset{(b)}{\leq}\ & \max_{a} \!\Big\{ \!\Big(\!\big|\!R(s_i,\!a)\!-\!R'(s_j',\!a)\!\big| \!+\!\gamma W_1\big(\mathbb{P}(\cdot|s_i,\!a), \mathbb{P}'(\cdot|s_j',\!a) ;\bar{d}\big)\!\Big)\notag\\
    \ &\quad- \!\Big(\!\big|\!R(s_i,\!a)\!-\!R'(s_j',\!a)\!\big| \!+\!\gamma W_1\big(\mathbb{P}(\cdot|s_i,\!a), \mathbb{P}'\!(\!\cdot|s_j',\!a) ;d_n\!\big)\!\Big) \!\Big\}\notag\\
    = &\  \gamma \max_{a} \Big\{W_1\big(\mathbb{P}(\cdot|s_i,a), \mathbb{P}'(\cdot|s_j',a) ;\bar{d}\big)  \notag\\
    &\ \qquad\quad - W_1\big(\mathbb{P}(\cdot|s_i,a), \mathbb{P}'(\cdot|s_j',a) ;d_n\big)\Big\}\notag\\
    \overset{(c)}{\leq }&\  \gamma \max_{a} \Big\{ \max_{i,j} \big\{\bar{d}(s_i, s_j')-d_n(s_i, s_j')\big\} \notag\\
    &\ \qquad\quad+W_1\big(\mathbb{P}(\cdot|s_i,a), \mathbb{P}'(\cdot|s_j',a) ;d_n\big)\notag\\
    &\ \qquad\quad- W_1\big(\mathbb{P}(\cdot|s_i,a), \mathbb{P}'(\cdot|s_j',a) ;d_n\big)\Big\}\notag\\
    = &\  \gamma \max_{i,j} \big\{\bar{d}(s_i, s_j')-d_n(s_i, s_j')\big\} \leq \frac{\gamma^{n+1} R_\text{max}}{1-\gamma}.
\end{align}
Here, step~$(a)$ follows the definitions of $\bar{d}$ and $d_n$, step~$(b)$ is derived from $\bar{d}\geq d_{n+1}$, and step~$(c)$ follows from the result in (\ref{appenBineq}). Following from (\ref{AppenB:eq1})-(\ref{AppenB:eq3}), $\bar{d}(s_i, s_j')-d_n(s_i, s_j')\leq\frac{\gamma^n R_\text{max}}{1-\gamma}$ holds for all $n\in\mathbb{N}$. 

\ifCLASSOPTIONcaptionsoff
  \newpage
\fi
\footnotesize
\bibliographystyle{IEEEtran}

\bibliography{IEEEabrv,IEEEexample}
\end{document}